\documentclass{article}

\usepackage[preprint]{neurips_2026}

\usepackage[utf8]{inputenc}
\usepackage[T1]{fontenc}
\usepackage{hyperref}
\usepackage{url}
\usepackage{booktabs}
\usepackage{amsfonts}
\usepackage{nicefrac}
\usepackage{microtype}
\usepackage{xcolor}
\usepackage{colortbl}  
\definecolor{starblue}{RGB}{52,120,200}   
\definecolor{llmred}{RGB}{210,90,90}     
\newcommand{\emc}[4]{\cellcolor{#4!#1!white}#2\tiny{$\pm$#3}}
\newcommand{\llmc}[2]{\cellcolor{llmred!#1!white}#2}

\usepackage{needspace}
\usepackage{amsmath}
\usepackage{amssymb}
\usepackage{amsthm}
\usepackage{graphicx}
\usepackage{subcaption}
\usepackage{algorithm}
\usepackage{algorithmic}
\usepackage{enumitem}
\usepackage{multirow}
\usepackage{wrapfig}
\usepackage{float}  
\usepackage{placeins}  
\makeatletter
\def\star@seqsplit@end{\star@seqsplit@end}
\def\star@seqsplit@loop#1{%
  \ifx\star@seqsplit@end#1%
  \else
    #1\allowbreak
    \expandafter\star@seqsplit@loop
  \fi}
\providecommand{\seqsplit}[1]{\star@seqsplit@loop#1\star@seqsplit@end}
\makeatother

\newtheorem{theorem}{Theorem}

\newtheorem{proposition}[theorem]{Proposition}
\newtheorem{corollary}[theorem]{Corollary}

\newtheorem{remark}{Remark}

\newcommand{\Agents}{\mathcal{A}}
\newcommand{\States}{\mathcal{S}}
\newcommand{\Tasks}{\mathcal{T}}
\newcommand{\Board}{\mathcal{B}}

\newcommand{\Expect}{\mathbb{E}}

\newcommand{\indicator}{\mathbf{1}}
\newcommand{\starname}{\textsc{Star}}
\newcommand{\sinit}{\textnormal{\textsc{init}}}
\newcommand{\ssucc}{\textnormal{\textsc{succ}}}
\newcommand{\sfail}{\textnormal{\textsc{fail}}}
\newcommand{\sblock}{\textnormal{\textsc{block}}}
\newcommand{\smiss}{\textnormal{\textsc{miss}}}
\newcommand{\shead}{\textnormal{\textsc{head}}}
\newcommand{\sfuse}{\textnormal{\textsc{fuse}}}
\newcommand{\ssp}{\textnormal{\textsc{sp}}}
\newcommand{\stp}{\textnormal{\textsc{tp}}}

\newcommand{\stopo}{\textnormal{\textsc{topo}}}

\title{\starname: Failure-Aware Markovian Routing for Multi-Agent Spatiotemporal Reasoning}

\author{
 \textbf{Ruiyi Yang\textsuperscript{1}},
 \textbf{Lihuan Li\textsuperscript{1}}, \\
 \textbf{Hao Xue\textsuperscript{1,2}},
 \textbf{Flora D. Salim\textsuperscript{1}},
\\
 \textsuperscript{1}University of New South Wales, \\
 \textsuperscript{2}The Hong Kong University of Science and Technology
(Guangzhou), \\
 \small{
   \textbf{Correspondence:} \href{mailto:ruiyi.yang@unsw.edu.au}{ruiyi.yang@unsw.edu.au},
   \href{mailto:flora.salim@unsw.edu.au}{flora.salim@unsw.edu.au}
 }
}

\begin{document}

\maketitle

\begin{abstract}
Compositional spatiotemporal reasoning often requires a system to invoke multiple heterogeneous specialists, such as geometric, temporal, topological, and trajectory agents.
A central question is how such a system should route among specialists when execution does not simply succeed or fail, but fails in qualitatively different ways.
Existing tool-augmented and multi-agent LLM systems typically leave this routing decision implicit in language generation, making recovery ad hoc, difficult to interpret, and hard to optimize. This paper presents \starname{} (\textbf{S}patio-\textbf{T}emporal \textbf{A}gent \textbf{R}outer), a failure-aware routing framework that externalizes inter-agent control as a state-conditioned transition policy over the current agent, task type, and typed execution status.
At the center of \starname{} is an agent routing matrix that combines expert-specified nominal routes with recovery transitions learned from execution traces.
Because the matrix conditions on distinct failure states, the router can respond differently to malformed outputs, missing dependencies, and tool--query mismatches, rather than collapsing them into a generic retry signal.
Specialists execute through a tool-grounded \emph{extract--compute--deposit} protocol and write intermediate results to a shared blackboard for downstream fusion. Results prove that retaining unsuccessful traces during training enlarges the support of the routing policy on error states, enabling recovery transitions that success-only training cannot represent.
Across three spatiotemporal benchmarks and eight backbone LLMs, \starname{} improves over multiple baselines with the clearest gains on queries whose execution deviates from the nominal routing path.
Router-specific ablations and recovery analyses further show that typed failure-aware routing, rather than specialist composition alone, is a key factor for these improvements.
\end{abstract}
\section{Introduction}
\label{sec:intro}

Spatiotemporal reasoning requires models to combine heterogeneous forms of computation, including geometric localization, temporal relation reasoning, route planning, trajectory analysis, and graph-structured inference~\citep{quan2025benchmarking,li2025stbench,ni2026streasoner}.
Complex queries often invoke these computations sequentially, where downstream steps depend strictly on intermediate results produced by earlier ones. Consequently, the core challenge lies not just in executing individual computations, but in determining \textit{which} specialist should act next as the execution unfolds. We study this under the lens of \emph{failure-aware specialist routing}: determining the optimal next specialist given the current agent, the task type, and, crucially, the observed execution outcome.

This question is especially important when execution deviates from the nominal path.
Failures in specialist pipelines are not monolithic: an agent may produce a malformed output, become blocked because a required upstream result is missing, or fail because its tool menu cannot resolve the query.
These cases require different routing decisions, such as retrying with a sibling specialist, rerouting to an upstream producer, or escalating to another agent family.
However, existing tool-augmented and multi-agent LLM systems often leave such control decisions implicit in language generation, making recovery difficult to inspect, optimize, and reuse across executions~\citep{wei2026beyond,li2024survey}.

\begin{figure}[!htbp]
\centering
\includegraphics[width=0.95\textwidth]{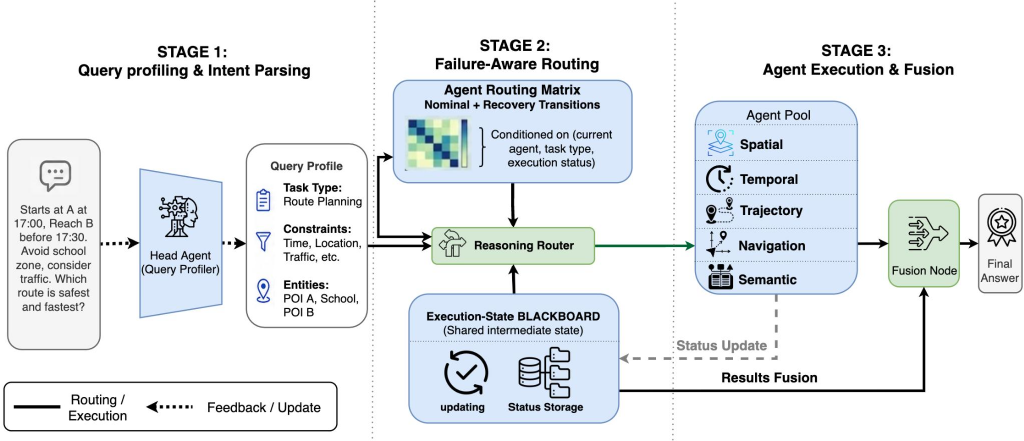}
\caption{\starname{} architecture. Queries are parsed into a task profile, the failure-aware routing matrix selects specialists conditioned on the current agent, task type, execution status, and specialists execute through an extract-compute-deposit protocol over a shared blackboard before final fusion.}
\label{fig:framework}
\end{figure}

This paper proposes \starname{} (\textbf{S}patio-\textbf{T}emporal \textbf{A}gent \textbf{R}outer), a failure-aware routing framework that externalizes inter-agent control into a state-conditioned transition policy.
At the center of \starname{} is an agent routing matrix indexed by the current agent, task type, execution status, and next agent.
The matrix combines expert-specified nominal routes for successful execution with recovery transitions learned from execution traces.
By conditioning on typed statuses such as malformed output, missing dependency, and tool--query mismatch, \starname{} treats recovery as an explicit routing problem rather than as an implicit byproduct of LLM generation. Specialists execute through a tool-grounded \emph{extract--compute--deposit} protocol.
The LLM selects a computation and extracts its parameters, deterministic tools perform the computation, and the resulting intermediate state is deposited into a shared blackboard for downstream agents and final fusion.
This protocol exposes typed execution outcomes to the router and provides the intermediate state needed for recovery transitions.

Our central claim is that typed execution feedback should be treated as a first-class control signal for multi-agent reasoning.
Structurally, retaining unsuccessful traces enlarges the support of the learned routing policy on error states, enabling recovery transitions that success-only training cannot represent.
Throughout experiments, \starname{} improves over LLM-only, prompting-based, and agentic baselines.
Router-specific ablations and recovery analyses further show that the gains are concentrated on executions that deviate from the nominal route, isolating the contribution of failure-aware routing rather than specialist composition alone. Our key \textbf{contributions} are:

\begin{enumerate}[leftmargin=*,itemsep=1pt]
    \item \textbf{Failure-aware specialist routing.}
     \starname{} formulates specialist composition as a state-conditioned routing problem over the current agent, task type, and typed execution status, making recovery decisions explicit rather than leaving them implicit in language generation.

    \item \textbf{An interpretable agent routing matrix.}
    A routing matrix is introduced combining expert nominal routes with learned recovery transitions from execution traces.
    The matrix conditions on distinct failure modes, including malformed outputs, missing dependencies, and tool--query mismatches.

    \item \textbf{Tool-grounded execution with shared intermediate state.}
    The routing policy is instantiated in a spatiotemporal multi-agent system where specialists follow an extract--compute--deposit protocol and write intermediate results to a shared blackboard for downstream fusion.

    \item \textbf{Structural and empirical evidence for recovery-aware routing.}
    Experiment results proved that retaining unsuccessful traces enables recovery transitions that success-only training cannot represent.
    Across three benchmarks and eight backbone LLMs, router-specific ablations and recovery analyses demonstrate that typed failure-aware routing improves robustness especially on executions that deviate from the nominal path.
\end{enumerate}
\section{Related Work}
\label{sec:related}

\paragraph{Tool-augmented reasoning and self-correction.}
Tool-augmented language models extend LLMs with access to external tools, APIs, code execution, and symbolic or numerical solvers during reasoning~\citep{schick2023toolformer,yao2022react,lu2023chameleon,qiao2023taskweaver,ferrag2025llm}.
These systems improve performance on tasks requiring computation, retrieval, planning, or environment interaction.
Related work on self-correction and reflection further improves reasoning by asking models to critique, revise, or retry unsuccessful outputs~\citep{renze2024self,madaan2023self,shinn2023reflexion}.
However, tool choice and recovery are often still handled through free-form generation, textual feedback, or iterative retries.

\paragraph{Multi-agent LLM systems and shared memory.}
Multi-agent LLM frameworks decompose complex tasks across role-specialized agents and coordinate them through dialogue, role prompting, message passing, or hand-authored workflows~\citep{wu2024autogen,hong2023metagpt,shen2025exploring,tran2025multi}.
Shared-memory and blackboard-style architectures allow agents to accumulate intermediate results in a common workspace~\citep{han2025exploring,salemi2025llm}.
These systems make intermediate reasoning more accessible, but recovery from intermediate failures is often treated as another round of interaction or as a manually specified fallback.

\paragraph{Learned routing and hierarchical control.}
Selecting among heterogeneous sub-behaviors has long been studied in hierarchical decision-making, options, modular policies, and mixture-of-experts routing~\citep{stolle2002learning,shazeer2017outrageously,habib2025llm,mei2025omnirouter,Yang2025RELOOP}.
These approaches separate high-level control from low-level execution, providing a useful perspective for routing among specialist modules.
In LLM-agent systems, however, routing decisions are often produced by latent generation behavior or task-level prompts rather than by an explicit execution-state transition model.

\paragraph{Spatiotemporal reasoning with language models.}
Recent benchmarks show that LLMs struggle with spatiotemporal reasoning across spatial relations, temporal ordering, forecasting, navigation, trajectories, and graph-structured inference~\citep{quan2025benchmarking,li2025stbench,ni2026streasoner}.
Prior studies have explored temporal reasoning, neuro-symbolic temporal reasoning, multimodal spatiotemporal reasoning, and spatial reasoning with LLMs~\citep{xiong2024large,ge2025tremu,cheng2025v,aghzal2023can, li2026zara}.
These tasks provide a natural testbed for specialist composition because they require heterogeneous computations and expose distinct execution failure modes.

\paragraph{Positioning.}
Our \starname{} lies at the intersection of tool-augmented reasoning, multi-agent coordination, and learned routing.
Its focus is not on adding more specialists or tools, but on making inter-agent control explicit through a failure-aware routing matrix.
By conditioning routing on the current agent, task type, and typed execution status, \starname{} separates the question of what went wrong from the question of which specialist should act next.
\section{Method}
\label{sec:method}

\starname{} is a failure-aware routing framework for specialist composition in spatiotemporal reasoning.
The key design is to separate specialist execution from inter-agent control: agents produce intermediate results and typed execution statuses, while the router decides which agent should act next.
This separation is necessary because different failures require different successors.
A malformed output, a missing upstream dependency, and a tool--query mismatch should not be collapsed into the same retry behavior.
Given a query $q$, a state-conditioned routing policy $\pi$ composes a pool of specialist agents $\Agents$ through a shared blackboard $\Board$, and the final answer is produced by the fusion agent:
\[
\mathcal{F}(q) = a_{\sfuse}(\Board_K, q), \qquad \Board_K = \mathcal{E}(q, \pi, \Agents).
\]
The central object in \starname{} is a routing matrix indexed by current agent, task type, execution status, and next agent.

\subsection{Failure-Aware Routing Formulation}
\label{sec:formulation}

We formalize specialist composition as routing over execution states:
\[
\mathcal{G} = (\Agents, \States, \Tasks, \pi, \Board, \Phi), \qquad
\pi: \Agents \times \States \times \Tasks \to \Delta(\Agents)
\]
where $\Agents$ is the agent set, $\Tasks$ is a finite task-type set, $\Board$ is the blackboard state space, and $\Phi=\{\phi_a\}$ is the set of per-agent execution functions.
The transition kernel $\pi$ maps the current agent, observed execution status, and task type to a distribution over next agents.

The execution-status set is $\States = \{\sinit,\ssucc,\sfail,\sblock,\smiss\}.$
Here, $\sinit$ and $\ssucc$ denote initialization and successful execution.
The three error states have distinct operational meanings:
$\sfail$ denotes a malformed result, $\sblock$ denotes a missing upstream dependency, and $\smiss$ denotes a tool--query mismatch.
Recovery is treated as a routing problem rather than free-form generation.

\subsection{State-Conditioned Routing Matrix}
\label{sec:routing}

The routing kernel combines stable nominal routes with learned recovery transitions:
\begin{equation}
\label{eq:dual_system}
\pi(a' \mid a,s,t)=
\begin{cases}
\pi_1(a' \mid a,s,t), & (a,s,t)\in\Omega_1,\\
\pi_2(a' \mid a,s,t), & \text{otherwise}.
\end{cases}
\end{equation}

For nominal states $s\in\{\sinit,\ssucc\}$, \textbf{System~1} follows an expert successor function $\sigma_t$ for each task type $t$: $
\pi_1(a' \mid a,s,t)=\indicator[a'=\sigma_t(a)].
$
For error states $s\in\{\sfail,\sblock,\smiss\}$, \textbf{System~2} uses a learned recovery matrix
$
\mathbf{M} \in \mathbb{R}^{N \times |\States| \times |\Tasks| \times N}_{\ge 0},
$
with
\begin{equation}
\label{eq:system2}
\pi_2(a' \mid a,s,t)=\mathbf{M}[a,s,t,a']
=\frac{\mathbf{C}[a,s,t,a']}{\sum_{a''}\mathbf{C}[a,s,t,a'']}.
\end{equation}

Each row $\mathbf{M}[a,s,t,\cdot]$ is a recovery profile: conditioned on the current agent $a$, task type $t$, and status $s$, it assigns mass to candidate successor agents.
Thus, the same agent and task type can route differently under $\sfail$, $\sblock$, and $\smiss$.

\begin{figure}[!htbp]
\centering
\includegraphics[width=\textwidth]{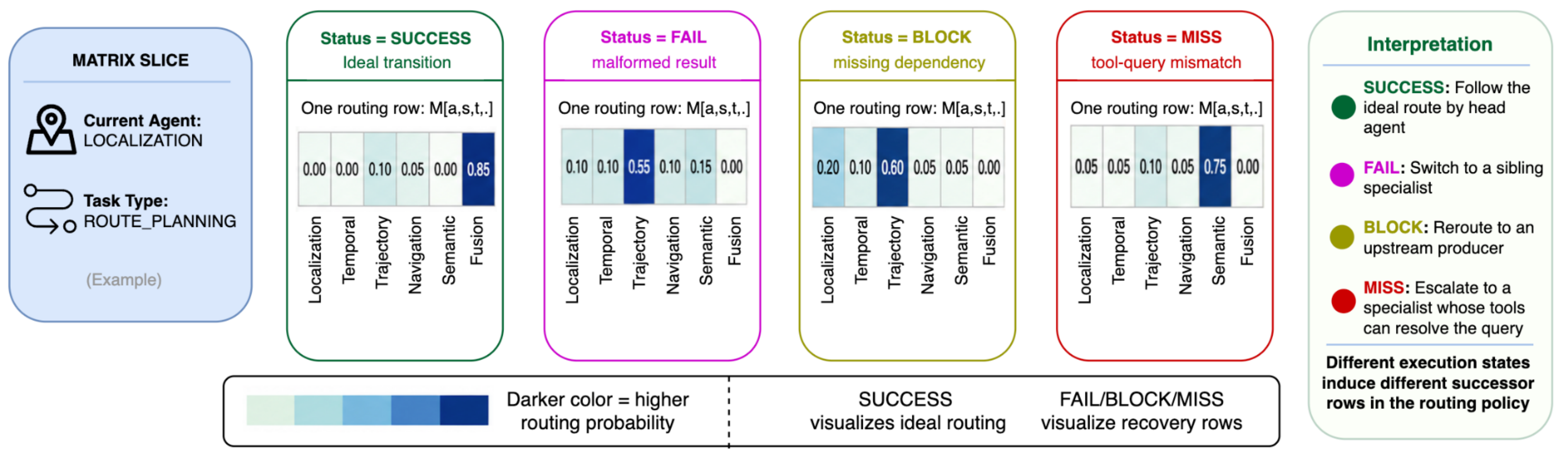}
\caption{State-conditioned routing matrix slices for a representative task type. Each panel visualizes a row $\mathbf{M}[a,s,t,\cdot]$ with fixed current agent $a$ and task type $t$ but different execution status $s$. The differing successor distributions show that recovery is conditioned on the type of failure rather than treated as a generic retry action.}
\label{fig:matrix_slice}
\vspace{-0.8em}
\end{figure}

\subsection{Specialist Execution via Extract--Compute--Deposit}
\label{sec:agent_pool}

Each agent $a \in \Agents^{(1)}$ exposes a computation menu $\mathcal{M}_a$ of deterministic functions and executes
\begin{equation}
\label{eq:agent_protocol}
\phi_a(\Board, q) = \bigl(s,\; \Board \cup \{(a, m^*(\theta))\}\bigr),
\end{equation}
where $m^*=\texttt{select}(q,\Board,\mathcal{M}_a)$ is the LLM-selected computation, $\theta=\texttt{extract}(q,\Board)$ are the extracted parameters, and $m^*(\theta)$ is the deterministic output.
The LLM handles computation selection and parameter extraction, while numerical, geometric, temporal, and graph operations are delegated to deterministic solvers.
Each specialist returns an updated blackboard and a status $s\in\States$, making execution outcomes observable to the routing matrix.
The HEAD agent produces the query profile, the FUSION agent synthesizes the final answer, and full menus are provided in Appendix~\ref{app:agents}.

\begin{figure}[!htbp]
\centering
\begin{minipage}{0.68\textwidth}
    \centering
    \includegraphics[width=\linewidth]{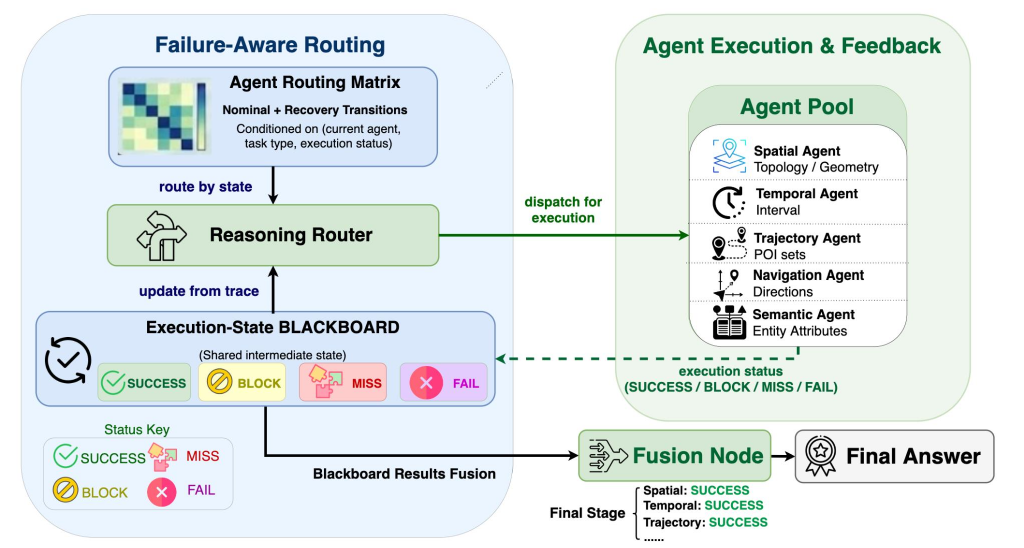}
\end{minipage}
\hfill
\begin{minipage}{0.28\textwidth}
    \small
    \textbf{Routing feedback loop.}
    The routing matrix dispatches specialists based on the current agent, task type, and execution status.
    After execution, specialists write typed status signals to the blackboard, which updates subsequent routing decisions.
\end{minipage}
\caption{Failure-aware routing and execution feedback in \starname{}. Specialists return typed signals such as \textsc{Success}, \textsc{Block}, \textsc{Miss}, and \textsc{Fail}; the execution-state blackboard stores these signals and supports subsequent routing and final fusion.}
\label{fig:routing_feedback}
\vspace{-0.8em}
\end{figure}

\subsection{Failure-Aware Matrix Training}
\label{sec:matrix_training}

For each training query $(q_i,y_i)$ with task type $t_i$, we execute the pipeline, record the trace
$
\xi_i=[(a^{(i)}_0,s^{(i)}_0),\ldots],
$
and determine final-answer correctness $r_i \in \{0,1\}$.
Observed transitions are weighted as
$
w(r_i)=r_i+\alpha(1-r_i), \qquad \alpha \in (0,1),
$
so successful traces receive weight $1$ and unsuccessful traces receive attenuated weight $\alpha$.
The weighted count tensor is
\begin{equation}
\label{eq:count_tensor}
\mathbf{C}[a,s,t,a']
=
\sum_i \sum_j
w(r_i)\,
\indicator\!\bigl[a^{(i)}_j=a,\; s^{(i)}_j=s,\; t_i=t,\; a^{(i)}_{j+1}=a'\bigr].
\end{equation}
Row-normalizing $\mathbf{C}$ yields the recovery matrix $\mathbf{M}$ in Eq.~\ref{eq:system2}.
When an agent enters an error state during training, candidate recovery specialists are also evaluated on the same query, and successful recoveries are added to the corresponding row of $\mathbf{C}$.
This augmentation is used only to construct the matrix; at inference time, routing is determined by $\mathbf{M}$ without oracle recovery signals.
A terminal constraint enforces $a_{\sfuse}$ as an absorbing state. The role of $\alpha$ is to retain information from unsuccessful traces so that recovery transitions remain observable.
The empirical precision--coverage trade-off of $\alpha$ is analyzed in Appendix~\ref{app:alpha}.

\subsection{Inference with Parallel Activation}
\begin{algorithm}[!htbp]
\caption{\starname{} inference}
\label{alg:execution}
\small
\begin{algorithmic}[1]
\REQUIRE Query $q$, kernel $\pi$, agents $\Agents$, threshold $\tau$, max steps $T$
\STATE $\Board_0 \gets \emptyset$, \; $\mathcal{R} \gets \emptyset$
\STATE $(t,p)\gets \phi_{a_{\shead}}(\Board_0,q)$; \; $\Board_1\gets \{(\shead,p)\}$; \; $a\gets a_{\shead}$; \; $s\gets \ssucc$
\FOR{$k=1,\ldots,T$}
    \STATE $\Agents_{\text{next}} \gets \{a' : \pi(a' \mid a,s,t)\ge \tau\} \setminus \mathcal{R}$
    \IF{$a_{\sfuse}\in\Agents_{\text{next}}$}
        \RETURN $a_{\sfuse}(\Board_k,q)$
    \ENDIF
    \STATE $\{(s_{a'},\Board^{a'})\}_{a'} \gets \texttt{scatter}\bigl(\{\phi_{a'}(\Board_k,q)\}_{a'\in\Agents_{\text{next}}}\bigr)$
    \STATE $\Board_{k+1}\gets \Board_k \cup \bigcup_{a'} \Board^{a'}$
    \STATE $\mathcal{R}\gets \mathcal{R}\cup\{a' : s_{a'}=\sfail\}$
    \STATE $(a,s)\gets \texttt{pivot}(\{(a',s_{a'})\}_{a'\in\Agents_{\text{next}}})$
\ENDFOR
\RETURN $a_{\sfuse}(\Board_T,q)$
\end{algorithmic}
\end{algorithm}
\vspace{-1.0em}

At inference time, the router activates all successor agents whose probabilities exceed a threshold $\tau$.
The system can execute multiple plausible successors in parallel when the routing matrix assigns mass to several candidates.
The activated agents run in a scatter--gather step, deposit their outputs into the blackboard, and return typed statuses.
A fixed pivot rule then selects the next control state from the returned statuses.
Error statuses are prioritized for recovery, while failed agents are retired to avoid repeated malformed executions. Detailed inference algorithm is shown in Algorithm~\ref{alg:execution}.

\begin{figure}[t]
\centering
\vspace{-0.5em}
\includegraphics[width=0.85\textwidth]{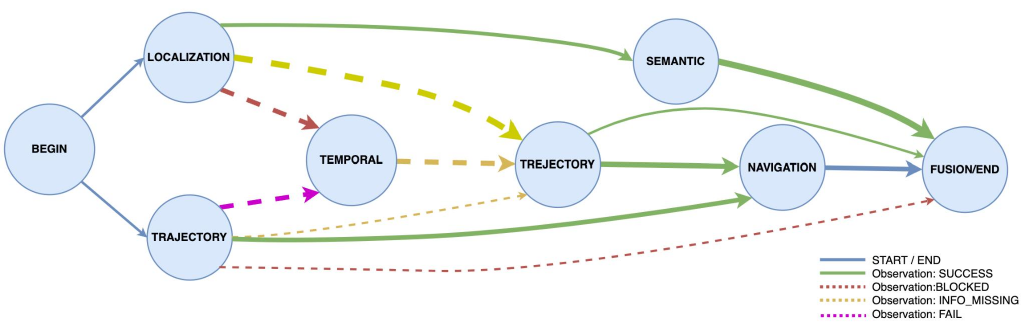}
\vspace{-0.5em}
\caption{Routing example through the dual-system kernel. System~1 nominal routes (green arrows) handle $\ssucc$ transitions, while System~2 learned recovery transitions (red arrows) handle $\sfail$, $\sblock$, and $\smiss$ by escalating to alternative specialists or rerouting to fill missing dependencies.}
\vspace{-1em}
\label{fig:routing}
\end{figure}

\subsection{Recovery Reachability}
\label{sec:theory}

Failure-aware training is necessary for representing recovery transitions in the learned routing matrix.
If training retains only successful traces, then error-state rows that are visited only by unsuccessful executions remain empty. Unsuccessful traces are not merely negative examples: they provide evidence needed to populate otherwise empty recovery rows of the routing matrix.
Theorem~\ref{thm:recovery}shows the recovery reachability dominance, while the full proof are deferred to Appendix~\ref{app:proofs}.

\begin{theorem}[Recovery Reachability Dominance]
\label{thm:recovery}
Let $\mathbf{M}^{\alpha}$ be the transition matrix trained with $w(r)=r+\alpha(1-r)$ for $\alpha>0$, and let $\mathbf{M}^{0}$ be the success-only matrix.
For any error state $(a,s,t)$ with $s\in\{\sfail,\sblock,\smiss\}$,
\[
\bigl|\,\textup{supp}\,\mathbf{M}^\alpha[a, s, t, \cdot]\,\bigr| \;\geq\; \bigl|\,\textup{supp}\,\mathbf{M}^0[a, s, t, \cdot]\,\bigr|,
\]
with strict inequality whenever at least one unsuccessful trace traverses $(a,s,t)$ with a successor not observed in successful traces.
\end{theorem}

\begin{proof}[Proof sketch]
From Eq.~\ref{eq:count_tensor}, each transition count is weighted by $w(r)=r+\alpha(1-r)$.
When $\alpha>0$, both successful and unsuccessful traces contribute nonzero weight, so any successor observed from an error state remains represented in the corresponding row of $\mathbf{M}^{\alpha}$ after normalization.
When $\alpha=0$, unsuccessful traces contribute zero weight, so successors observed only in unsuccessful executions disappear from $\mathbf{M}^{0}$.
The support containment follows immediately, and the inequality is strict whenever an unsuccessful trace contributes a successor not observed in successful traces.
\end{proof}

\section{Experiments}
\label{sec:experiments}

\subsection{Setup}
\label{sec:setup}

Three spatiotemporal benchmarks are evaluated: STARK~\citep{quan2025benchmarking} (26 task types), STBench (henceforth \textbf{STB24})~\citep{li2025stbench} (11 task types), and ST-Bench~(2026) (henceforth \textbf{STB26})~\citep{ni2026streasoner} (4 task types). We use eight backbone LLMs spanning proprietary and open-source families from 3B to 20B parameters.
The primary metric is exact match (EM) with numeric tolerance, while regression tasks are evaluated with RMSE/MAE.
Baselines include LLM-only, LLM-only with extended reasoning, Reflexion~\citep{shinn2023reflexion}, ReAct~\citep{yao2022react}, Tree-of-Thought~\citep{yao2023tree}, Graph-of-Thought~\citep{besta2024graph, yao2024got}, and function-calling.
Unless otherwise noted, the routing matrix uses $\alpha=0.3$ and $\tau=0.4$.
Full implementation details, prompts, and per-task results are deferred to Appendices~\ref{app:benchmarks}--\ref{app:full_results}.

\subsection{Main Results}
\label{sec:results}

\setlength{\intextsep}{2pt}\setlength{\columnsep}{8pt}
\needspace{10\baselineskip}
\begin{wraptable}{r}{0.50\textwidth}
\centering
\caption{\starname{} vs.\ prompting and agent baselines (Qwen3-8B, EM-eligible queries; 95\% Wilson CIs).}
\label{tab:baselines}
\scriptsize
\setlength{\tabcolsep}{3pt}
\begin{tabular}{lccc}
\toprule
Method & STARK & STB24 & STB26 \\
\midrule
\textbf{\starname{} (ours)}    & \textbf{73.0}\tiny{$\pm$1.6} & \textbf{76.3}\tiny{$\pm$1.8} & \textbf{55.3}\tiny{$\pm$2.9} \\
Reflexion                      & 68.8\tiny{$\pm$2.1} & 37.2\tiny{$\pm$2.9} & 40.0\tiny{$\pm$3.5} \\
ReAct                          & 65.6\tiny{$\pm$2.3} & 31.8\tiny{$\pm$2.6} & 26.7\tiny{$\pm$3.1} \\
Tree-of-Thought                & 62.8\tiny{$\pm$2.4} & 11.8\tiny{$\pm$2.1} & 40.0\tiny{$\pm$3.5} \\
Graph-of-Thought               & 61.7\tiny{$\pm$2.4} & 10.2\tiny{$\pm$5.7} & 42.5\tiny{$\pm$3.7} \\
LLM-only (CoT)                 & 59.4\tiny{$\pm$1.8} & 34.2\tiny{$\pm$1.9} & 45.0\tiny{$\pm$2.9} \\
LLM-only                       & 55.4\tiny{$\pm$1.8} & 31.9\tiny{$\pm$1.8} & 37.5\tiny{$\pm$2.8} \\
Function-calling               & 31.9\tiny{$\pm$2.1} & 20.8\tiny{$\pm$2.3} & 33.3\tiny{$\pm$3.0} \\
\bottomrule
\end{tabular}
\end{wraptable}
Table~\ref{tab:baselines} compares \starname{} against prompting and agent baselines on the Qwen3-8B backbone.
\starname{} outperforms Reflexion ($+$4.2/$+$39.1/$+$15.3pp on STARK/STB24/STB26), ReAct, Tree-of-Thought, Graph-of-Thought, plain LLM-only ($+$17.6/$+$44.4/$+$17.8pp), LLM-only with extended reasoning, and a function-calling baseline on all three benchmarks, with the largest gains on STB24 where structured computation is especially important.
Additional per-task comparisons, including against published STB24 baselines, are reported in Appendix~\ref{app:published}.

\begin{table}[!htbp]
\centering
\caption{Per-benchmark EM (\%) across eight backbones on EM-eligible queries.  For each backbone we report LLM-only (red columns) and \starname{} (blue columns) per benchmark; cell shading is proportional to the value, so a darker blue cell next to a lighter red cell indicates a larger \starname{} lift.  Five numeric prediction types (two STARK, two STB24, one STB26) are excluded from EM and reported in Table~\ref{tab:regression_metrics}; \starname{} ``$\pm$'' values are 95\% Wilson half-widths (App.~\ref{app:stat_sig}).}
\label{tab:main}
\scriptsize
\setlength{\tabcolsep}{3pt}
\begin{tabular}{ll ccc ccc cc}
\toprule
 & & \multicolumn{3}{c}{\emph{LLM-only}} & \multicolumn{3}{c}{\emph{\starname{} (ours)}} & & \\
\cmidrule(lr){3-5} \cmidrule(lr){6-8}
Model & Params & STARK & STB24 & STB26 & STARK & STB24 & STB26 & Lat & Tok \\
\midrule
\multicolumn{10}{l}{\emph{Proprietary}} \\
Claude Sonnet 4.6  & --  & \llmc{81}{71.4} & \llmc{45}{45.2} & \llmc{30}{30.1} & \emc{69}{\textbf{79.3}}{2.7}{starblue} & \emc{83}{\textbf{83.3}}{2.4}{starblue} & \emc{65}{\textbf{64.5}}{5.6}{starblue} & 14.0 & 5{,}500 \\
Claude Haiku 4.5   & --  & \llmc{70}{70.3} & \llmc{29}{28.7} & \llmc{46}{45.6} & \emc{65}{74.6}{2.8}{starblue} & \emc{76}{76.1}{2.6}{starblue} & \emc{55}{55.0}{5.8}{starblue} & 8.9  & 4{,}073 \\
\midrule
\multicolumn{10}{l}{\emph{Open-source}} \\
GPT-OSS-20B        & 20B & \llmc{88}{67.9} & \llmc{51}{51.4} & \llmc{59}{47.2} & \emc{75}{\textbf{75.1}}{1.6}{starblue} & \emc{72}{71.8}{1.8}{starblue} & \emc{47}{59.2}{2.9}{starblue} & 7.1 & 5{,}312 \\
Qwen3-8B           & 8B  & \llmc{55}{55.4} & \llmc{32}{31.9} & \llmc{38}{37.5} & \emc{73}{73.0}{1.6}{starblue} & \emc{76}{\textbf{76.3}}{1.8}{starblue} & \emc{55}{\textbf{55.3}}{2.9}{starblue} & 1.6 & 4{,}191 \\
Ministral-3-8B     & 8B  & \llmc{51}{51.3} & \llmc{32}{32.1} & \llmc{42}{42.0} & \emc{64}{64.1}{1.8}{starblue} & \emc{67}{66.8}{1.9}{starblue} & \emc{46}{46.2}{2.9}{starblue} & 1.9 & 4{,}625 \\
Llama-3.1-8B       & 8B  & \llmc{50}{50.4} & \llmc{32}{31.5} & \llmc{39}{38.6} & \emc{68}{67.8}{1.7}{starblue} & \emc{48}{48.2}{2.0}{starblue} & \emc{44}{43.9}{2.9}{starblue} & 2.2 & 4{,}161 \\
GLM-4-9B           & 9B  & \llmc{53}{52.9} & \llmc{37}{37.0} & \llmc{48}{45.6} & \emc{65}{65.3}{1.8}{starblue} & \emc{48}{48.4}{2.0}{starblue} & \emc{46}{48.3}{2.9}{starblue} & 1.8 & 3{,}746 \\
Llama-3.2-3B       & 3B  & \llmc{42}{41.8} & \llmc{22}{22.3} & \llmc{32}{32.0} & \emc{52}{51.8}{1.9}{starblue} & \emc{43}{43.1}{2.0}{starblue} & \emc{39}{39.1}{2.8}{starblue} & 1.3 & 4{,}013 \\
\bottomrule
\end{tabular}
\end{table}

Table~\ref{tab:main} reports per-benchmark EM for both LLM-only and \starname{} across all eight backbones, with cell shading on each side proportional to the value. \starname{} consistently outperforms LLM-only on every benchmark; Qwen3-8B provides a strong efficiency--accuracy trade-off at substantially lower latency than GPT-OSS-20B and is used as the default backbone for all subsequent ablations.

\subsection{Routing Mechanism Analysis}
\label{sec:router_ablation}

This section asks whether the proposed routing formulation is necessary, not merely whether the broader system helps.
Framework-level ablations (Table~\ref{tab:framework_ablation}) remove non-routing components---tools, blackboard, HEAD classification---while router-specific ablations (Table~\ref{tab:router_ablation}) vary only the routing kernel and are the more informative test of state-conditioned routing itself.

Table~\ref{tab:framework_ablation} shows that the non-routing scaffolding is load-bearing.
Removing the blackboard causes the largest drop, especially on STARK, confirming that shared intermediate state is critical for fusion.
Removing deterministic tools substantially hurts STARK, while simple routing without HEAD classification reduces accuracy and increases token cost.
These results show that tools, blackboard accumulation, and HEAD routing are all important.
\begin{table}[!htbp]
\centering
\caption{Framework-level ablations.  Each row removes a non-routing component from the full system.}
\label{tab:framework_ablation}
\small
\begin{tabular}{lp{4.5cm}ccc}
\toprule
Config & Description & STARK & STB24 & STB26 \\
\midrule
\starname{} (full)  & Complete system                                & \textbf{73.0}\tiny{$\pm$3.1} & \textbf{76.3}\tiny{$\pm$3.5} & \textbf{55.3}\tiny{$\pm$7.4} \\
Simple routing      & No HEAD classification                         & 68.3\tiny{$\pm$3.2} & 65.7\tiny{$\pm$3.9} & 50.4\tiny{$\pm$7.4} \\
No tools            & Agents reason without deterministic solvers    & 45.1\tiny{$\pm$3.5} & 49.8\tiny{$\pm$4.2} & 40.7\tiny{$\pm$7.3} \\
No blackboard       & Agents run but $\Board$ not passed to FUSION   & 21.6\tiny{$\pm$2.9} & 46.6\tiny{$\pm$4.2} & 46.8\tiny{$\pm$7.4} \\
LLM-only            & No routing formulation, chain-of-thought only  & 59.4\tiny{$\pm$3.4} & 34.2\tiny{$\pm$4.0} & 45.0\tiny{$\pm$7.4} \\
\bottomrule
\end{tabular}
\end{table}

\begin{table}[t]
\centering
\caption{Router ablations (Qwen3-8B).  \emph{Overall EM} reports accuracy on all queries, in Difference positive = \starname{} better.}
\label{tab:router_ablation}
\small
\setlength{\tabcolsep}{4pt}
\begin{tabular}{l ccc ccc}
\toprule
 & \multicolumn{3}{c}{\emph{Overall EM}} & \multicolumn{3}{c}{\emph{Difference}} \\
\cmidrule(lr){2-4} \cmidrule(lr){5-7}
Configuration & STARK & STB24 & STB26 & STARK & STB24 & STB26 \\
\midrule
\starname{} (baseline)        & 73.0 & 76.3 & 55.3 & --   & --    & --   \\
System 1 only                 & 63.0 & 63.1 & 50.4 & $+$10.0 & $+$13.2 & $+$4.9  \\
System 2 only                 & 61.5 & 73.3 & 50.7 & $+$11.5 & $+$5.8 & $+$4.6 \\
No status conditioning        & 54.7 & 68.6 & 55.3 & $+$18.3 & $+$7.7 & n/a \\
$\alpha = 0$ (success-only)   & 68.2 & 76.3 & 54.0 & $+$0.0 & $+$7.1 & n/a \\
No recovery augmentation      & 67.9 & 76.3 & 55.3 & $+$5.6 & $+$12.5 & n/a \\
Random routing                & 48.0 & 44.4 & 50.7 & $+$27.0 & $+$29.5 & $+$4.6 \\
Force-semantic                & 54.2 & 52.3 & 52.7 & $+$13.4 & $+$25.6 & $+$0.5 \\
Monolithic tool agent         & 38.0 & 21.8 & 26.0 & $+$23.2 & $+$47.8 & $+$24.5 \\
\bottomrule
\end{tabular}
\end{table}

The ablation in Table~\ref{tab:router_ablation} further isolate the role of each routing component.
System~1 carries much of STARK, while System~2 is especially important on STB24: disabling System~2 leaves STARK unchanged but costs 13.3pp on STB24's queries, whereas disabling System~1 costs 14.0pp on STARK and 5.8pp on STB24.
Status conditioning is also important: collapsing \sfail, \sblock, and \smiss{} into a single status costs 33.3pp on STARK and 7.7pp on STB24.
Both asymmetric weighting and recovery augmentation matter: $\alpha=0$ costs 7.1pp on STB24, and removing recovery augmentation costs 5.6/12.5pp on STARK/STB24.
Random routing, force-semantic routing, and the monolithic tool agent provide strong floor references, dropping 13--48pp on the diverging subset.

\subsection{Failure Recovery}
\label{sec:recovery}

\begin{table}[!htbp]
\centering
\caption{Worked example of autonomous failure recovery on a STARK \texttt{LANDMARK\_DIRECTION} query (id \texttt{stark\_direction\_questions\_37}; full trace in App.~\ref{app:case_d}).  The primary specialist cannot resolve named landmarks without a gazetteer and emits $\smiss$; the router consults $\mathbf{M}[\text{LM\_DIR}, \ssp, \smiss, \cdot]$ and dispatches \sfuse{} directly.  No human intervention or pre-coded fallback is involved---the recovery transition is entirely learned from training traces.}
\label{tab:recovery_example}
\small
\setlength{\tabcolsep}{4pt}
\begin{tabular}{c l p{8.7cm}}
\toprule
Step & Component & Action / outcome \\
\midrule
1 & \shead{}    & Classifies query as \texttt{STARK\_LANDMARK\_DIRECTION}; System~1 dispatches \ssp{}. \\
2 & \ssp{}      & Tool \texttt{landmark\_direction} cannot resolve ``Alcatraz Island'', ``Union Square'' to coordinates without an external gazetteer $\Rightarrow$ returns \smiss{}. \\
3 & Router      & System~1 exhausted; System~2 looks up $\mathbf{M}[\text{LM\_DIR}, \ssp, \smiss, \cdot]$.  Top successor is \sfuse{} (learned: world-knowledge LLM answers post-tool-failure). \\
4 & \sfuse{}    & Reads raw query + ``tool could not resolve landmarks'' flag; uses LLM world knowledge that Union Square is south of Alcatraz $\Rightarrow$ emits \texttt{[1.0]} = \textbf{GT}. \\
\bottomrule
\end{tabular}
\end{table}

Theorem~\ref{thm:recovery} shows that failure-aware training enlarges the support of the routing matrix on error states; here we examine what those additional transitions do at inference time.
Table~\ref{tab:recovery_example} traces a single concrete recovery: the primary specialist hits a tool limitation, returns a typed-failure status ($\smiss$), and the matrix autonomously reroutes to a successor that the training traces identified as effective for that exact \emph{(task, agent, status)} triple. Crucially, the dispatch decision is made by the learned $\mathbf{M}$, not by a hand-coded fallback table.
Table~\ref{tab:recovery_pathways} shows that the three error states correspond to distinct recovery channels.
The strongest result is $\smiss$ recovery.
On STARK, queries that enter the $\smiss$ state recover to 80.8\% EM, exceeding the no-failure baseline by 5.5 percentage points, while STB24 $\smiss$ recovery reaches 82.1\%, exceeding the no-failure baseline by 4.6 percentage points.
This indicates that learned recovery can do more than provide generic fallback: when the initial specialist is mismatched, the learned transition can reroute the query to a more suitable agent. In other words, the recovery policy is not merely repairing broken executions; in some cases it corrects an initially suboptimal specialist assignment and routes the query onto a more appropriate computation path.
By contrast, \sfail{} recovery is substantially harder, suggesting that residual errors are dominated by parameter extraction rather than routing.
\sblock{} corresponds to rerouting for missing dependencies; its detailed breakdown is deferred to Appendix~\ref{app:recovery}.

\begin{table}[!htbp]
\centering
\caption{The three failure-recovery channels learned by $\mathbf{M}$ with (Qwen3-8B).  Each error state is a first-class routing pathway with its own learned successor distribution.}
\label{tab:recovery_pathways}
\small
\begin{tabular}{l p{4.6cm} c c}
\toprule
Error state & Recovery mechanism & STARK & STB24 \\
\midrule
\smiss{} (tool menu mismatch) & Escalate to a specialist whose menu can resolve the query, typically the semantic / LLM-reasoning agent. & \textbf{80.8\%} & \textbf{82.1\%} \\
\sfail{} (malformed result)   & Invoke a sibling specialist of the same family with re-extracted parameters.                       & 51.7\% & 56.6\% \\
\sblock{} (missing dependency) & Reroute to the upstream agent that produces the missing blackboard entry, then retry.            & 63.2\% & 67.7\% \\
\midrule
\multicolumn{2}{l}{\emph{No-failure baseline (queries that never enter an error state)}} & 75.3\% & 77.5\% \\
\bottomrule
\end{tabular}
\end{table}

Taken together, the theorem, the $\alpha=0$ and no-recovery-augmentation ablations in Table~\ref{tab:router_ablation}, and the recovery rates in Table~\ref{tab:recovery_pathways} support the value of failure-aware routing from complementary perspectives: representability, causal contribution, and observed behavior.

\subsection{Cross-Benchmark Transfer and Composite Queries}
\label{sec:generalization}

Table~\ref{tab:cross} probes whether the learned routing matrix captures reusable routing structure across benchmarks.
Each row applies a matrix learned from one source benchmark to all three targets, while the first row reports the benchmark-specific matrix used on its own target.

\paragraph{Cross-benchmark transfer.}
\setlength{\intextsep}{2pt}\setlength{\columnsep}{8pt}
\begin{wraptable}{r}{0.46\textwidth}
\centering
\caption{Cross-benchmark transfer (Qwen3-8B).  Each row uses one source-benchmark matrix to route queries on all three targets.}
\label{tab:cross}
\scriptsize
\setlength{\tabcolsep}{4pt}
\begin{tabular}{lccc}
\toprule
Source $\to$ target     & STARK & STB24 & STB26 \\
\midrule
Baseline (each-on-self) & 73.0 & 76.3 & 55.3 \\
STARK matrix $\to$ all  & 73.0 & 71.1 & 50.9 \\
STB24 matrix $\to$ all  & 69.4 & 76.3 & 48.4 \\
STB26 matrix $\to$ all  & 68.8 & 69.5 & 55.3 \\
\bottomrule
\end{tabular}
\end{wraptable}
The strongest average transfer source is the STARK matrix.
It preserves its own STARK performance and transfers to STB24 and STB26 with moderate drops of 5.2pp and 4.4pp, respectively.
This suggests that the STARK matrix captures routing patterns that are broadly useful across spatial, temporal, trajectory, and navigation-style tasks. The STB24 and STB26 matrices also preserve their own benchmark performance. The STB24 matrix transfers moderately to STARK, with a 3.6pp drop, but loses 6.9pp on STB26.
Conversely, the STB26 matrix preserves STB26 performance but drops by 4.2pp on STARK and 6.8pp on STB24.
These results indicate that the learned routing matrix is not purely benchmark-specific, but transfer depends on how well the source benchmark covers the target's task types and recovery states. Overall, cross-benchmark transfer reveals both reusability and boundary conditions of failure-aware routing.
Routing patterns transfer best when source and target share overlapping computational primitives and similar specialist dependencies.
Transfer degrades when the target requires recovery rows or task-specific transitions that are underrepresented in the source matrix.
\paragraph{Composite queries.}
As a qualitative probe of parallel specialist composition, we hand-designed composite queries across three categories (event analysis, trajectory monitoring, logistics) and extended the HEAD agent to output a list of required sub-types, triggering parallel activation in the routing kernel.
Composite \starname{} achieves 50.0\% overall and 66.7\% partial, vs.\ ReAct 35.7/50.0, Reflexion 28.6/42.9, and LLM-only 21.4/35.7.
The strongest finding is on trajectory monitoring (4 queries), which requires the TRAJECTORY and SPATIAL agents running in parallel: composite \starname{} achieves 75\% while all baselines score 0\% because sequential single-agent loops cannot decompose into two specialists simultaneously.
We treat this as qualitative evidence that the parallel-activation mechanism composes naturally with multi-type decomposition, not as a quantitative result; 14 queries are too few for a major claim.
Full traces and tables appear in Appendix~\ref{app:composite}.

\section{Conclusion}
\label{sec:conclusion}

\starname{} is a failure-aware state-conditioned routing formulation for specialist composition.
Its central argument is that compositional reasoning over heterogeneous specialists requires a routing policy conditioned on \emph{what went wrong}: \sfail, \sblock, and \smiss demand semantically distinct recovery, and success-only routing is structurally unable to learn any of them (Theorem~\ref{thm:recovery}).
On three spatiotemporal benchmarks and eight backbones, the routing formulation consistently improves per-benchmark EM over LLM-only baselines and outperforms Reflexion, ReAct, and other prompting and agent baselines.
Router-specific ablations isolate the contribution of each routing component and a dedicated recovery analysis provides the sharpest evidence: STARK queries entering the $\smiss$ state recover at 80.8\%, above the no-failure baseline, suggesting that recovery transitions rescue mis-routed queries rather than provide generic fallback.
Future work targets the parameter-extraction step that bounds \sfail{} recovery and learned refinements of the routing policy beyond trace-count estimation.

\paragraph{Limitations.}
STAR currently learns routing matrices from execution traces within a given task taxonomy.
As the transfer experiments show, routing generalizes best when source and target benchmarks share computational primitives and specialist dependencies.
Future work will study task-agnostic recovery abstractions and larger composite-query evaluations.

\begin{ack}
This research is partially supported by the ARC Training Centre for Whole Life Design of Carbon Neutral Infrastructure (IC230100015). We express our gratitude to Sharon AI for providing access to NVIDIA
H100 GPUs.
\end{ack}

\bibliographystyle{plain}
\bibliography{reference}

\newpage
\appendix
\section{Theoretical Proofs and Structural Properties}
\label{app:proofs}

This appendix collects proofs, structural properties, and auxiliary analysis for \starname{}.
Theorem~\ref{thm:recovery} is stated in the main text, while its full proof and related structural results are given here.
We also provide additional analysis of blackboard accumulation, routing length, the precision--coverage trade-off of $\alpha$, and connections to hierarchical control.

\subsection{Proof of Theorem~\ref{thm:recovery}}
\label{app:proof_recovery}

\begin{proof}[Proof of Theorem~\ref{thm:recovery}]
By Eq.~(\ref{eq:count_tensor}), the weighted count tensor under parameter $\alpha$ is
\[
\mathbf{C}^{\alpha}[a,s,t,a']
=
\sum_i \sum_j
\bigl(r_i+\alpha(1-r_i)\bigr)\,
\indicator\!\bigl[a^{(i)}_j=a,\; s^{(i)}_j=s,\; t_i=t,\; a^{(i)}_{j+1}=a'\bigr].
\]
For $\alpha>0$, both successful traces ($r_i=1$) and unsuccessful traces ($r_i=0$) contribute non-zero weight.
Hence, for any fixed error-state triple $(a,s,t)$ with $s\in\{\sfail,\sblock,\smiss\}$, the support of $\mathbf{C}^{\alpha}[a,s,t,\cdot]$ includes every successor observed in either successful or unsuccessful traces traversing that state.
After row normalization, the same support is inherited by $\mathbf{M}^{\alpha}[a,s,t,\cdot]$.

In contrast, when $\alpha=0$, unsuccessful traces contribute zero weight, so
\[
\mathbf{C}^{0}[a,s,t,a']
=
\sum_i \sum_j
r_i\,
\indicator\!\bigl[a^{(i)}_j=a,\; s^{(i)}_j=s,\; t_i=t,\; a^{(i)}_{j+1}=a'\bigr].
\]
Therefore, the support of $\mathbf{M}^{0}[a,s,t,\cdot]$ contains only successors observed in successful traces.

It follows immediately that
\[
\bigl|\,\textup{supp}\, \mathbf{M}^{\alpha}[a,s,t,\cdot]\bigr|
\ge
\bigl|\,\textup{supp}\, \mathbf{M}^{0}[a,s,t,\cdot]\bigr|.
\]
Moreover, if at least one unsuccessful trace traverses $(a,s,t)$ and contributes a successor not observed in successful traces, then the containment is strict, yielding
\[
\bigl|\,\textup{supp}\, \mathbf{M}^{\alpha}[a,s,t,\cdot]\bigr|
>
\bigl|\,\textup{supp}\, \mathbf{M}^{0}[a,s,t,\cdot]\bigr|.
\]
This proves the claim.
\end{proof}

\subsection{Blackboard Filtration and Information Non-Decrease}
\label{app:thm_info}

\begin{theorem}[Information Non-Decrease]
\label{thm:info}
Let $Y$ be the answer random variable and $\Board_k$ the blackboard state at step $k$.  The blackboard filtration satisfies $I(Y;\, \Board_k) \leq I(Y;\, \Board_{k+1})$ for all $0 \leq k < K$.
\end{theorem}

\begin{proof}
Let $Y$ denote the answer random variable and $\Board_k$ the blackboard at step $k$.  By construction the blackboard is append-only, so $\Board_k \subseteq \Board_{k+1}$.  Write $\Board_{k+1} = (\Board_k, \Delta_k)$ where $\Delta_k = \Board_{k+1} \setminus \Board_k$.

By the chain rule for mutual information:
\begin{equation}
    I(Y;\, \Board_{k+1}) = I(Y;\, \Board_k, \Delta_k) = I(Y;\, \Board_k) + I(Y;\, \Delta_k \mid \Board_k)
\end{equation}

Since mutual information is non-negative, $I(Y;\, \Delta_k \mid \Board_k) \geq 0$, and therefore:
\begin{equation}
    I(Y;\, \Board_{k+1}) \geq I(Y;\, \Board_k)
\end{equation}

Equality holds iff $\Delta_k \perp Y \mid \Board_k$---the new outputs are conditionally independent of the answer given existing evidence.  By induction, $I(Y;\, \Board_0) \leq \cdots \leq I(Y;\, \Board_K)$.
\end{proof}

\begin{corollary}[Fusion Sufficiency]
\label{cor:fusion}
The fusion agent operates on $\Board_K \supseteq \{(a, m^*_a(\theta_a))\}$ for all executed agents $a$.  Therefore:
\begin{equation}
    I(Y;\, \Board_K) \geq I(Y;\, \phi_a(\Board_0, q)) \quad \forall\, a \text{ executed}
\end{equation}
The multi-agent pipeline provides at least as much information about $Y$ as any single-agent baseline with access to only one agent's output.
\end{corollary}

\begin{remark}[Sufficiency vs.\ exploitation]
While $\Board_K$ is a sufficient statistic among all functions of the agent outputs, the LLM-based fusion function $F(\Board_K, q)$ may not exploit all information.  The gap $I(Y; \Board_K) - I(Y; F(\Board_K, q))$ is the residual that motivates future work on fusion-stage improvement.
\end{remark}

\subsection{Bounded Routing Length}
\label{app:thm_convergence}

\begin{theorem}[Bounded Convergence]
\label{thm:convergence}
The routing process in Algorithm~\ref{alg:execution} terminates in at most $\min(T, |\Agents^{(1)}| + 2)$ steps.
\end{theorem}

\begin{proof}
Termination in $K^* \leq \min(T, |\Agents^{(1)}| + 2)$ steps is shown as follows.

\textbf{Case 1 (Nominal).}  Under System~1, $\sigma_t$ defines an acyclic path from $a_{\shead}$ through a subset of $\Agents^{(1)}$ to $a_{\sfuse}$.  Path length $\leq |\Agents^{(1)}| + 2$.

\textbf{Case 2 (With retirements).}  Each agent producing \sfail enters $\mathcal{F}$ (Algorithm~\ref{alg:execution}, line~11) and is excluded from $\Agents_\text{next}$ (line~5).  Each retirement eliminates $\geq 1$ agent.  After at most $|\Agents^{(1)}|$ retirements, the safety net routes to $a_{\sfuse}$.

\textbf{Case 3 (Budget).}  The explicit cap $T$ ensures termination regardless.

Combined: $K^* \leq \min(T, |\Agents^{(1)}| + 2) = \min(T, 8)$ with the \starname{} agent pool.
\end{proof}

\begin{proposition}[Expected Step Count]
\label{prop:steps}
Let $p_s$ denote the average specialist success rate and $d$ the path length for task type $t$.  The expected routing steps are:
\begin{equation}
    \Expect[K] \approx d + 2 + (1-p_s) \cdot \bar{R}
\end{equation}
where $\bar{R}$ is the mean recovery path length.  With empirical rates $p_s \approx 0.82$ (STARK), $0.65$ (STB24), $0.55$ (STB26), $\Expect[K] \in [3.2, 4.8]$---well below the bound of~8.
\end{proposition}

\subsection{Precision--Coverage Trade-off of $\alpha$}
\label{app:alpha}

Theorem~\ref{thm:recovery} shows that $\alpha > 0$ is necessary for the matrix to have non-empty recovery support on error states.
This subsection characterizes the trade-off introduced by $\alpha$: as more unsuccessful traces are retained, recovery coverage improves while nominal-path precision may degrade.

\begin{proposition}[Precision--Coverage Trade-off of $\alpha$]
\label{prop:alpha}
Define routing precision and recovery coverage as
\begin{align}
\text{Prec}(\alpha) &= \Expect_{(a,s,t) \sim \Omega_1}\bigl[\pi_2(a^* \mid a, s, t)\bigr], \\
\text{Cov}(\alpha) &= \frac{1}{|\Omega_2|}\sum_{(a,s,t) \in \Omega_2} \indicator\bigl[\exists\, a': \mathbf{M}^\alpha[a,s,t,a'] > 0\bigr],
\end{align}
where $a^*$ denotes the optimal successor under an oracle policy and $\Omega_2$ the set of error-state triples.
Then $\text{Prec}(\alpha)$ is non-increasing in $\alpha$, $\text{Cov}(\alpha)$ is non-decreasing, and the parameter $\alpha$ smoothly interpolates between maximum routing precision ($\alpha \to 0$, zero recovery coverage) and maximum recovery coverage ($\alpha \to 1$, noisy nominal routing).
\end{proposition}

\begin{proof}[Proof sketch]
Increasing $\alpha$ strictly enlarges the support of every error-state row (Theorem~\ref{thm:recovery}), proving the coverage monotonicity.
For precision, consider a nominal state $(a, s, t) \in \Omega_1$ where the optimal successor $a^*$ is the most-observed successor in successful traces; as $\alpha$ grows, unsuccessful-trace transitions to non-optimal successors $a' \neq a^*$ receive relatively more mass in the row normalization, decreasing the normalized probability of $a^*$.
\end{proof}

\begin{figure}[H]
\centering
\includegraphics[width=0.85\textwidth]{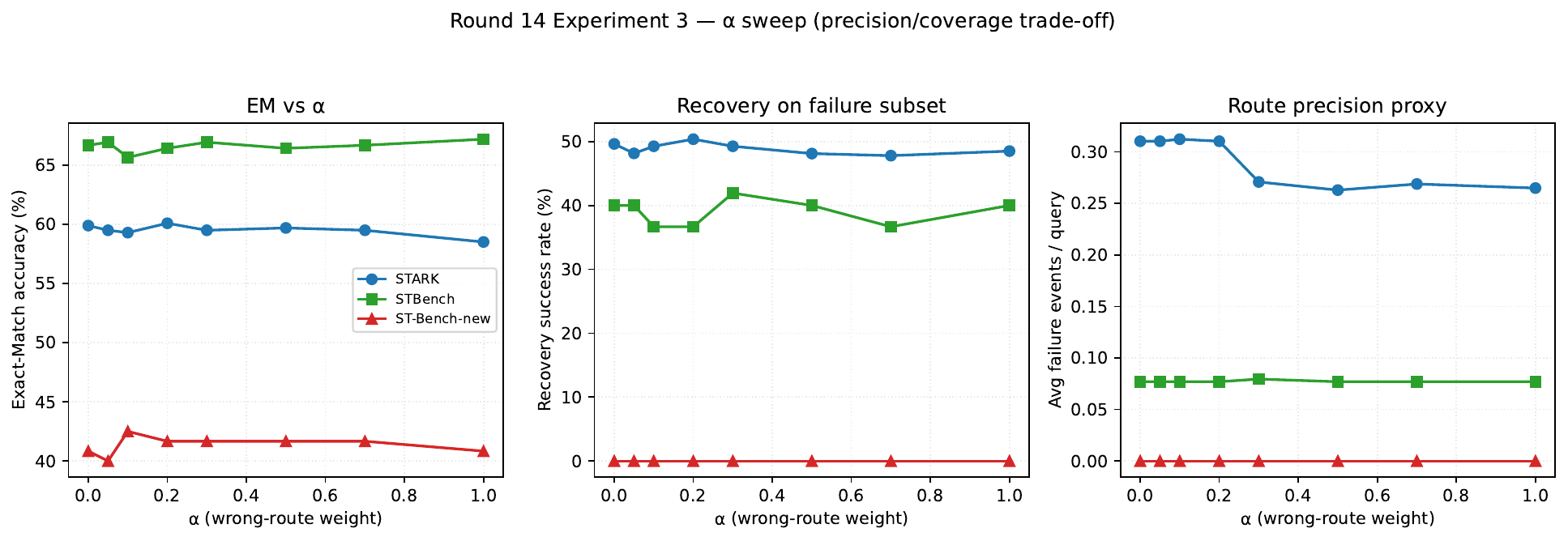}
\caption{Empirical precision--coverage trade-off of $\alpha$ across three benchmarks (Qwen3-8B, balanced slice).  Sweeping $\alpha \in [0, 1]$ confirms Proposition~\ref{prop:alpha}: recovery coverage is non-decreasing in $\alpha$ while overall accuracy degrades only gently over the useful range.  The operating point $\alpha = 0.3$ used throughout the main text sits in the plateau region where recovery transitions are populated without materially perturbing nominal-path routing.  The $\alpha = 0$ configuration (success-only training) corresponds to the empty-support regime of Theorem~\ref{thm:recovery}.}
\label{fig:alpha_sweep}
\end{figure}

\subsection{Parallel Execution Speedup}
\label{app:parallel}

\begin{proposition}[Scatter-Gather Speedup]
\label{prop:speedup}
At step $k$, let $m = |\Agents_k|$ agents execute with individual times $\{T_a\}$.  Sequential execution requires $T_\text{seq} = \sum_a T_a$.  Parallel scatter-gather requires:
\begin{equation}
    T_\text{par} = \max_a T_a + T_\text{oh}, \qquad S = \frac{T_\text{seq}}{T_\text{par}} \leq m
\end{equation}
where $T_\text{oh}$ is scheduling overhead.  Since deterministic computations execute in $<$10ms and LLM inference dominates at 1--5s, $T_\text{oh} \ll T_a$ and speedup approaches $m$ for homogeneous latencies.
\end{proposition}

\subsection{Extract--Compute--Deposit as Information Channel}
\label{app:channel}

Each agent $a$ can be modeled as a noisy channel from query $q$ to blackboard entry $b_a$:
\begin{equation}
    q \xrightarrow[\text{(noisy)}]{\texttt{extract}} \theta \xrightarrow[\text{(exact)}]{\texttt{compute}} m^*(\theta) = b_a
\end{equation}
The extraction step is the noisy component (the LLM may extract incorrect parameters); the computation step is noiseless (deterministic solver).  The channel capacity is:
\begin{equation}
    C_a = \max_{p(q)} I(q;\, b_a)
\end{equation}
Since $b_a = m^*(\theta)$ is deterministic, $H(b_a \mid \theta) = 0$, and so $I(q; b_a) = I(q; \theta)$ when $m^*$ is injective.  The bottleneck is \emph{entirely in extraction}: improving parameter extraction directly improves channel capacity, while computation accuracy is guaranteed by construction.  This formalizes why \starname{} achieves near-perfect accuracy on computation-intensive tasks---the extraction problem (``which parameters?'') is much simpler than the original computation problem.

\subsection{Connection to Hierarchical Reinforcement Learning}
\label{app:hrl}

\starname{}'s routing admits a natural interpretation through the lens of options.  Each agent $a$ corresponds to an option with:
\begin{itemize}
    \item \textbf{Initiation set}: $\mathcal{I}_a = \{(a',s,t): \pi(a \mid a',s,t) > 0\}$
    \item \textbf{Policy}: the extract--compute--deposit protocol $\phi_a$
    \item \textbf{Termination}: single-step (the agent always terminates after one execution)
\end{itemize}
The kernel $\pi$ is a \emph{policy over options}, and the dual-system decomposition combines a hand-crafted policy (System~1) with a learned one (System~2).

Unlike standard options, \starname{}'s agents are deterministic in their computation and stochastic only in parameter extraction (the LLM component).  This separation suggests the matrix could be further refined via reinforcement learning---optimizing transition probabilities directly against task accuracy---bridging trace-based training and reward-based optimization.

\FloatBarrier
\section{STB24 Comparison with Published Baselines}
\label{app:published}

Table~\ref{tab:published} compares \starname{} (Qwen3-8B, 8B parameters, local) against published GPT-4o results on STB24 from~\citep{li2025stbench}.
Open-7B$^\ddagger$ denotes the best reported open-source 7B-class model from the same source.

\begin{table}[!htbp]
\centering
\caption{STB24: \starname{} vs.\ published baselines. $\uparrow$/$\downarrow$ denotes higher/lower better.}
\label{tab:published}
\small
\begin{tabular}{lcccc}
\toprule
Task ($\uparrow$ unless noted) & \starname{} & GPT-4o & ChatGPT & Open-7B$^\ddagger$ \\
\midrule
Direction Det.     & \textbf{100.0} & 54.3  & 17.0  & 19.7 \\
Traj.\ Region      & \textbf{81.7}  & 11.0  & 27.9  & 27.9 \\
Traj.\ Anomaly     & \textbf{67.4}  & 60.2  & 53.8  & 51.0 \\
Point Region       & \textbf{92.4}  & 91.9  & 92.4  & 90.4 \\
Flow (MAE$\downarrow$) & \textbf{36.3} & 43.2 & 37.3 & 31.9 \\
Traj.\ Pred.\ (m$\downarrow$) & \textbf{112.6} & --    & --    & 139.4 \\
\midrule
Point-Traj.        & \textbf{100.0} & --    & 75.2  & 55.4 \\
POI Category       & 90.4           & \textbf{95.9} & 79.3 & 44.6 \\
Admin Region       & 89.1           & \textbf{96.6} & 83.6 & 30.1 \\
Navigation         & 62.5           & \textbf{75.5} & 43.8 & 38.9 \\
Urban Region       & 47.6           & \textbf{60.3} & 39.8 & 27.0 \\
\bottomrule
\end{tabular}
\end{table}

\starname{} outperforms GPT-4o on 4 of 9 accuracy tasks and on both regression metrics, with the advantage concentrating on computation-intensive tasks where deterministic solvers guarantee exact answers.
GPT-4o retains the lead on knowledge-intensive tasks (admin region, urban function) where world knowledge matters.

\FloatBarrier
\section{Recovery Analysis: Extended Breakdown}
\label{app:recovery}

Table~\ref{tab:recovery_full} reports the recovery behavior by first-agent execution status, complementing the benchmark-level recovery summary in Table~\ref{tab:recovery_pathways}.
For each benchmark, we partition test queries by the first error status observed during execution and compute the final EM after System~2 recovery transitions and fusion.
The ``no failure'' row reports the final EM on queries that never enter an error state and serves as the within-benchmark reference point.

\begin{table}[!htbp]
\centering
\caption{Recovery analysis by first-agent status on the Qwen3-8B balanced slice. ``No failure'' denotes queries that never enter an error state. ``Recovered EM'' is the final-answer accuracy after the corresponding error state is observed and the recovery policy is applied.}
\label{tab:recovery_full}
\small
\begin{tabular}{l l r c c}
\toprule
Benchmark & First-agent status & $n$ & Recovered EM (\%) & vs.\ no-failure \\
\midrule
\multirow{4}{*}{STARK}
    & no failure (baseline) & 631 & 75.3 & -- \\
    & $\sfail$              & 155  & 51.7 & $-$19.6 \\
    & $\smiss$              & 73  & 80.8 & $+$5.5 \\
    & $\sblock$             & 49  & 63.2 & $-$12.1 \\
\midrule
\multirow{4}{*}{STB24}
    & no failure (baseline) & 555 & 77.5 & -- \\
    & $\sfail$              & 122  & 56.6 & $-$20.8 \\
    & $\smiss$              & 67  & 82.1 & $+$4.7 \\
    & $\sblock$             & 31  & 67.7 & $-$9.8 \\
\bottomrule
\end{tabular}
\end{table}

The clearest natural recovery signal comes from the $\smiss$ state.
On STARK, queries entering $\smiss$ recover to 80.8\% EM, 5.5 percentage points above the no-failure baseline of 75.3\%.
On STB24, $\smiss$ recovery reaches 82.1\%, 4.7 percentage points above the no-failure baseline of 77.5\%.
This pattern suggests that tool--query mismatches are often recoverable by rerouting to a more suitable specialist.
Rather than merely repairing a broken execution, the recovery transition can correct an initially suboptimal agent assignment and move the query onto a better computation path.

Recovery from $\sfail$ is substantially harder.
STARK $\sfail$ queries recover to 51.7\% EM, while STB24 $\sfail$ queries recover to 56.6\% EM.
This indicates that malformed outputs are less easily repaired by routing alone, likely because they often originate from incorrect parameter extraction, invalid intermediate structure, or incompatible output formatting.
Overall, the extended breakdown shows that typed error states induce different recovery profiles.
$\smiss$ is the most recoverable natural failure mode and can exceed the no-failure baseline because rerouting may correct an initially mismatched specialist choice.
$\sfail$ remains the hardest because routing cannot always repair malformed intermediate outputs.
$\sblock$ is best understood as a controlled stress-test pathway for missing-dependency recovery.
These observations provide empirical support for using typed execution feedback as a first-class routing signal.

\FloatBarrier
\section{Benchmark Details}
\label{app:benchmarks}

\subsection{Task Type Taxonomy}

\begin{table}[!htbp]
\centering
\caption{Complete task type taxonomy (35 types across three benchmarks).}
\label{tab:taxonomy}
\scriptsize
\begin{tabular}{llll}
\toprule
Benchmark & Tier & Task Type & Description \\
\midrule
\multirow{18}{*}{STARK}
 & T1 & SPATIAL\_IMPUTE & Spatial gap filling \\
 & T1 & TEMPORAL\_IMPUTE & Temporal gap filling \\
 & T1 & SPATIOTEMPORAL\_IMPUTE & Joint ST gap filling \\
 & T1 & SPATIAL\_LOCALIZATION & Where is an entity \\
 & T1 & TEMPORAL\_LOCALIZATION & When did an event occur \\
 & T1 & SPATIAL\_TRACKING & Multi-step spatial tracking \\
 & T1 & TEMPORAL\_TRACKING & Multi-step temporal tracking \\
 & T1 & SPATIOTEMPORAL\_FORECAST & Joint ST forecasting \\
 & T2 & SPATIAL\_RELATIONSHIP & Spatial relation between entities \\
 & T2 & TEMPORAL\_RELATIONSHIP & Temporal relation between events \\
 & T2 & SPATIOTEMPORAL\_RELATIONSHIP & Joint ST relation \\
 & T3 & LANDMARK\_PROXIMITY & Distance to landmark \\
 & T3 & LANDMARK\_DIRECTION & Direction to landmark \\
 & T3 & INTENT\_PREDICTION & Predict agent intent \\
 & T3 & POI\_PREDICTION & Predict point of interest \\
 & T3 & ROUTE\_PLANNING & Plan a route \\
 & T3 & ROUTE\_SEGMENT\_DURATION & Duration of route segment \\
 & T3 & ETA\_CALCULATION & Estimated time of arrival \\
\midrule
\multirow{13}{*}{STB24}
 & Comp. & DIRECTION\_DETERMINATION & Compute direction \\
 & Comp. & FLOW\_PREDICTION & Predict flow values \\
 & Comp. & NAVIGATION & Route planning \\
 & Reas. & ADMIN\_REGION & Identify admin.\ region \\
 & Reas. & POINT\_REGION & Point-in-region test \\
 & Reas. & POINT\_TRAJECTORY & Point-trajectory distance \\
 & Reas. & TRAJECTORY\_REGION & Trajectory-region membership \\
 & App. & TRAJECTORY\_ANOMALY & Detect anomalies \\
 & App. & TRAJECTORY\_PREDICTION & Predict trajectory \\
 & App. & TRAJECTORY\_TRAJECTORY & Trajectory-trajectory relation \\
 & Know. & POI\_CATEGORY\_RECOGNITION & POI category \\
 & Know. & URBAN\_REGION\_FUNCTION & Urban function \\
 & Gen. & GENERAL & General QA \\
\midrule
\multirow{4}{*}{STB26}
 & -- & CORRELATION & Correlation analysis \\
 & -- & ENTITY & Entity identification \\
 & -- & ETIOLOGICAL & Causal reasoning \\
 & -- & FORECASTING & Time-series forecasting \\
\bottomrule
\end{tabular}
\end{table}

\subsection{Dataset Statistics}

\begin{table}[!htbp]
\centering
\caption{Dataset splits (80/20 train/test re-split).}
\label{tab:splits}
\small
\begin{tabular}{lccc}
\toprule
& Task Types & Test & Source \\
\midrule
STARK & 18 & 2{,}912 & prquan/STARK\_10k \\
STB24 & 13 & 2{,}995 & STB24 paper \\
STB26 & 4 & 1{,}245 & Time-HD-Anonymous/ST-Bench \\
\midrule
\textbf{Total} & \textbf{35} & \textbf{7{,}152} & \\
\bottomrule
\end{tabular}
\end{table}

\FloatBarrier
\section{Cross-Model Comparison}
\label{app:multi_model}

\begin{table}[!htbp]
\centering
\caption{Per-benchmark EM (\%) across eight backbone models under \starname{} (EM-eligible queries, $N=6{,}363$).}
\label{tab:multi_model}
\small
\begin{tabular}{llccc}
\toprule
Model & Params & STARK & STB24 & STB26 \\
\midrule
\multicolumn{5}{l}{\emph{Proprietary}} \\
Claude Sonnet 4.6  & --  & 69.3          & \textbf{83.3} & \textbf{64.5} \\
Claude Haiku 4.5   & --  & 64.6          & 76.1          & 55.0 \\
\midrule
\multicolumn{5}{l}{\emph{Open-source}} \\
GPT-OSS-20B & 20B & \textbf{75.1} & 71.8          & 47.2 \\
Qwen3-8B           & 8B  & 73.0          & 70.9          & 48.9 \\
Ministral-3-8B     & 8B  & 64.1          & 66.8          & 46.2 \\
Llama-3.1-8B       & 8B  & 67.8          & 48.2          & 43.9 \\
GLM-4-9B           & 9B  & 65.3          & 48.4          & 45.6 \\
Llama-3.2-3B       & 3B  & 51.8          & 43.1          & 39.1 \\
\bottomrule
\end{tabular}
\end{table}

\begin{table}[!htbp]
\centering
\caption{Latency and cost comparison across backbone models.}
\label{tab:latency_models}
\small
\begin{tabular}{lcccccc}
\toprule
Model & Params & Avg Lat (s) & Avg Tokens & LLM/q \\
\midrule
GPT-OSS-20B     & 20B & 7.2 & 5{,}355 & 2.7 \\
Qwen3-8B        & 8B  & 1.8 & 4{,}272 & 2.6 \\
Ministral-3-8B  & 8B  & 2.1 & 4{,}762 & 2.7 \\
Llama-3.1-8B    & 8B  & 2.5 & 4{,}281 & 2.9 \\
GLM-4-9B        & 9B  & 2.0 & 3{,}810 & 2.6 \\
Llama-3.2-3B    & 3B  & 1.3 & 4{,}021 & 2.8 \\
\bottomrule
\end{tabular}
\end{table}

\FloatBarrier
\section{Full Per-Task-Type Results}
\label{app:full_results}

\begin{table*}[!htbp]
\centering
\caption{Per-task results on STARK (Qwen3-8B, $N=2{,}912$; EM=73.0\% on $N_\text{em}=2{,}792$ EM-eligible queries, excluding SPATIOTEMPORAL\_FORECAST and TEMPORAL\_TRACKING).  EM is reported with 95\% Wilson CIs.  $^\star$: numeric prediction task; EM is not meaningful (free-form real-valued output almost never matches the gold string under exact-match normalization) and is therefore omitted; accuracy is reported via RMSE/RMSPE in Table~\ref{tab:regression_metrics} and in the main text.  Lat: mean latency.  Tok: mean total tokens.  LLM/q: mean LLM calls per query.}
\label{tab:stark_full}
\scriptsize
\begin{tabular}{llccccc}
\toprule
Tier & Task Type & $N$ & EM (\%) & Lat (s) & Tokens & LLM/q \\
\midrule
T1 & SPATIAL\_IMPUTE              & 60  & \textbf{100.0}\tiny{$\pm$0.0} & 1.6 & 3{,}587 & 2.5 \\
T1 & TEMPORAL\_LOCALIZATION       & 60  & \textbf{98.3}\tiny{$\pm$3.3}  & 1.0 & 2{,}765 & 2.0 \\
T1 & TEMPORAL\_IMPUTE             & 60  & 90.0\tiny{$\pm$7.5}           & 1.5 & 4{,}174 & 3.0 \\
T1 & SPATIAL\_LOCALIZATION        & 300 & 58.0\tiny{$\pm$5.6}           & 1.3 & 3{,}246 & 2.9 \\
T1 & SPATIOTEMPORAL\_IMPUTE       & 60  & 55.0\tiny{$\pm$12.2}          & 1.7 & 5{,}655 & 3.0 \\
T1 & SPATIAL\_TRACKING            & 300 & 31.3\tiny{$\pm$5.2}           & 2.8 & 4{,}500 & 2.6 \\
T1 & SPATIOTEMPORAL\_FORECAST$^\star$  & 60  & --                       & 2.6 & 7{,}886 & 4.0 \\
T1 & TEMPORAL\_TRACKING$^\star$       & 60  & --                       & 0.6 & 3{,}086 & 2.0 \\
\midrule
T2 & TEMPORAL\_RELATIONSHIP       & 260 & \textbf{100.0}\tiny{$\pm$0.7} & 0.4 & 1{,}647 & 1.0 \\
T2 & SPATIAL\_RELATIONSHIP        & 680 & 97.9\tiny{$\pm$1.1}           & 2.3 & 4{,}531 & 3.0 \\
T2 & SPATIOTEMPORAL\_RELATIONSHIP & 630 & 72.1\tiny{$\pm$3.5}           & 4.0 & 6{,}234 & 3.0 \\
\midrule
T3 & LANDMARK\_DIRECTION          & 20  & 55.0\tiny{$\pm$20.0}  & 0.8 & 3{,}474 & 3.0 \\
T3 & ETA\_CALCULATION             & 26  & 50.0\tiny{$\pm$17.9}  & 0.6 & 2{,}022 & 3.0 \\
T3 & LANDMARK\_PROXIMITY          & 96  & 49.0\tiny{$\pm$9.8}   & 1.4 & 3{,}243 & 4.0 \\
T3 & ROUTE\_SEGMENT\_DURATION     & 20  & 30.0\tiny{$\pm$18.7}  & 2.3 & 3{,}748 & 3.8 \\
T3 & ROUTE\_PLANNING              & 20  & 60.0\tiny{$\pm$19.7}  & 1.0 & 4{,}627 & 3.5 \\
T3 & POI\_PREDICTION              & 100 & 52.0\tiny{$\pm$9.6}   & 6.9 & 8{,}512 & 3.6 \\
T3 & INTENT\_PREDICTION           & 100 & 42.0\tiny{$\pm$9.5}   & 6.2 & 8{,}166 & 3.5 \\
\midrule
& \textbf{EM-eligible TOTAL}       & 2{,}792 & \textbf{73.0}\tiny{$\pm$1.6} & -- & 4{,}191 & 2.6 \\
\bottomrule
\end{tabular}
\end{table*}

\begin{table}[!htbp]
\centering
\caption{Regression metrics for the five EM-excluded numeric-prediction task types (Qwen3-8B). These rows correspond to the $^\star$-marked entries in Tables~\ref{tab:stark_full}, \ref{tab:stbench_full}, and~\ref{tab:stbenchnew_full}. RMSPE is in percent.  TRAJECTORY\_PREDICTION uses Haversine ground distance (m).}
\label{tab:regression_metrics}
\small
\setlength{\tabcolsep}{4pt}
\begin{tabular}{llccccc}
\toprule
Benchmark & Task Type & $N$ & RMSE & MAE & AbsErr (m) & RMSPE (\%) \\
\midrule
STARK           & SPATIOTEMPORAL\_FORECAST & 60  & 34.05   & --     & --     & 123.6 \\
STARK           & TEMPORAL\_TRACKING       & 60  & 3.89    & --     & --     & --    \\
STB24         & FLOW\_PREDICTION         & 368 & 40.99   & 36.35  & --     & --    \\
STB24         & TRAJECTORY\_PREDICTION   & 184 & 0.001   & --     & 112.6  & --    \\
STB26 & FORECASTING              & 117 & 114.55  & 99.21  & --     & --    \\
\bottomrule
\end{tabular}
\end{table}

\begin{table*}[!htbp]
\centering
\caption{Per-category results on STB24 (Qwen3-8B, $N=2{,}995$; EM=70.9\% on $N_\text{em}=2{,}443$ EM-eligible queries, excluding FLOW\_PREDICTION and TRAJECTORY\_PREDICTION; 95\% Wilson CIs).  $^\star$: numeric prediction task; EM omitted and regression metric reported separately in Table~\ref{tab:regression_metrics}.  TRAJECTORY\_PREDICTION's native numeric-tolerance EM would score $\sim$100\% because of a coarse lat/lon tolerance, but the ground-distance AbsErr is 112.6\,m, so EM is not a meaningful accuracy measure for this task.}
\label{tab:stbench_full}
\scriptsize
\begin{tabular}{llccccccccc}
\toprule
Cat. & Task Type & $N$ & EM (\%) & MAE & RMSE & AbsErr (m) & Lat (s) & Tokens & LLM/q \\
\midrule
Comp.  & DIRECTION\_DET.          & 184 & \textbf{100.0}\tiny{$\pm$1.0} & --    & 0.000  & --     & 0.8 & 1{,}716 & 2.0 \\
App.   & TRAJ.\_PREDICTION$^\star$ & 184 & --                           & --    & 0.001  & 112.6  & 0.4 & 1{,}492 & 2.0 \\
Reas.  & POINT\_TRAJECTORY        & 184 & \textbf{100.0}\tiny{$\pm$1.0} & --    & 0.000  & --     & 2.6 & 2{,}859 & 2.0 \\
Reas.  & POINT\_REGION            & 184 & 92.4\tiny{$\pm$3.9}           & --    & 0.000  & --     & 2.9 & 5{,}094 & 4.0 \\
Know.  & POI\_CAT.\_RECOG.        & 146 & 90.4\tiny{$\pm$4.8}           & --    & --     & --     & 0.6 & 2{,}742 & 3.0 \\
Reas.  & ADMIN\_REGION            & 184 & 89.1\tiny{$\pm$4.5}           & --    & 0.000  & --     & 1.6 & 1{,}740 & 2.2 \\
Reas.  & TRAJ.\_REGION            & 230 & 81.7\tiny{$\pm$5.0}           & --    & --     & --     & 6.9 & 6{,}667 & 2.7 \\
App.   & TRAJ.\_ANOMALY           & 184 & 67.4\tiny{$\pm$6.7}           & --    & --     & --     & 0.2 & 1{,}610 & 1.0 \\
Comp.  & NAVIGATION               & 184 & 62.5\tiny{$\pm$6.9}           & --    & --     & --     & 0.5 & 2{,}019 & 3.0 \\
Know.  & URBAN\_REG.\_FUNC.       & 63  & 47.6\tiny{$\pm$12.0}          & --    & --     & --     & 0.6 & 3{,}456 & 3.0 \\
App.   & TRAJ.\_TRAJECTORY        & 553 & 44.5\tiny{$\pm$4.1}           & --    & --     & --     & 0.7 & 2{,}988 & 2.1 \\
Comp.  & FLOW\_PREDICTION$^\star$ & 368 & --                           & 36.35 & 40.987 & --     & 0.3 & 1{,}191 & 1.0 \\
Gen.   & GENERAL                  & 347 & 55.9\tiny{$\pm$5.2}           & --    & 1.000  & --     & 1.2 & 3{,}941 & 3.0 \\
\midrule
& \textbf{EM-eligible TOTAL}  & 2{,}443 & \textbf{70.9}\tiny{$\pm$1.8} & -- & -- & -- & -- & 4{,}272 & 2.6 \\
\bottomrule
\end{tabular}
\end{table*}

\begin{table}[!htbp]
\centering
\caption{Per-type results on STB26 (Qwen3-8B, $N=1{,}245$; EM=48.9\% on $N_\text{em}=1{,}128$ EM-eligible queries, excluding FORECASTING; 95\% Wilson CIs).  RMSE/MAE reported for FORECASTING.  $^\star$: numeric prediction task, EM not meaningful.}
\label{tab:stbenchnew_full}
\small
\begin{tabular}{lccccccc}
\toprule
Type & $N$ & EM (\%) & RMSE & MAE & Lat (s) & Tokens & LLM/q \\
\midrule
ETIOLOGICAL  & 82  & 62.2\tiny{$\pm$10.3} & --     & --    & 0.8 & 6{,}635 & 3.0 \\
CORRELATION  & 589 & 54.8\tiny{$\pm$4.0}  & --     & --    & 0.9 & 7{,}350 & 3.0 \\
ENTITY       & 457 & 38.9\tiny{$\pm$4.5}  & --     & --    & 0.9 & 6{,}950 & 3.0 \\
FORECASTING$^\star$ & 117 & --            & 114.55 & 99.21 & 0.2 & 1{,}684 & 1.0 \\
\midrule
\textbf{EM-eligible TOTAL} & 1{,}128 & \textbf{48.9}\tiny{$\pm$2.9} & -- & -- & -- & 4{,}272 & 2.6 \\
\bottomrule
\end{tabular}
\end{table}

\FloatBarrier
\section{Cross-Benchmark Generalization: Full Per-Type Results}
\label{app:cross_benchmark}

\begin{table*}[!htbp]
\centering
\caption{Cross-benchmark generalization: per-type comparison.  ``Specific'': benchmark-specific matrix.  ``Cross'': STARK-trained matrix applied to all benchmarks.  ``Classified As'': how the HEAD agent maps the original type to a STARK type.}
\label{tab:cross_full}
\scriptsize
\setlength{\tabcolsep}{2pt}
\begin{tabular}{l l cc c l cc}
\toprule
Bench. & Task Type & Cross & Spec. & $\Delta$ & Classified As & Lat (s) & Tok \\
\midrule
\multicolumn{8}{l}{\textbf{STARK} (67.9\% cross vs.\ 73.0\% specific)} \\
& TEMPORAL\_RELATIONSHIP       & 100.0 & 100.0 & +0.0 & TEMPORAL\_RELATIONSHIP       & 0.4 & 1{,}647 \\
& SPATIAL\_RELATIONSHIP        & 97.9  & 97.9  & +0.0 & SPATIAL\_RELATIONSHIP        & 2.3 & 4{,}529 \\
& TEMPORAL\_LOCALIZATION       & 98.3  & 98.3  & +0.0 & TEMPORAL\_LOCALIZATION       & 1.0 & 2{,}765 \\
& SPATIAL\_IMPUTE              & 100.0 & 100.0 & +0.0 & SPATIAL\_IMPUTE              & 1.6 & 3{,}575 \\
& TEMPORAL\_IMPUTE             & 90.0  & 90.0  & +0.0 & SPATIOTEMPORAL\_IMPUTE       & 1.5 & 4{,}175 \\
& SPATIAL\_LOCALIZATION        & 58.0  & 58.0  & +0.0 & SPATIAL\_LOCALIZATION        & 1.4 & 3{,}246 \\
& ST\_RELATIONSHIP             & 72.1  & 72.1  & +0.0 & ST\_RELATIONSHIP             & 4.0 & 6{,}234 \\
& ROUTE\_SEG.\_DUR.            & 40.0  & 30.0  & +10.0 & ST\_RELATIONSHIP            & 2.3 & 3{,}778 \\
\midrule
\multicolumn{8}{l}{\textbf{STB24} (61.1\% cross vs.\ 70.9\% specific)} \\
& DIRECTION\_DET.              & \textbf{100.0} & 100.0 & +0.0  & LANDMARK\_DIRECTION          & 0.8 & 1{,}994 \\
& TRAJ.\_PREDICTION$^\star$    & --             & --    & --    & ST\_FORECAST                 & 2.2 & 3{,}935 \\
& POINT\_TRAJECTORY            & \textbf{100.0} & 100.0 & +0.0  & SPATIAL\_RELATIONSHIP        & 2.6 & 3{,}156 \\
& POINT\_REGION                & 91.8  & 92.4  & $-$0.6 & SPATIAL\_RELATIONSHIP        & 3.0 & 5{,}406 \\
& ADMIN\_REGION                & \textbf{89.1}  & 89.1  & $-$0.1 & SPATIAL\_LOCALIZATION        & 1.5 & 2{,}006 \\
& TRAJ.\_REGION                & 81.7  & 81.7  & $-$0.1 & SPATIAL\_RELATIONSHIP        & 6.8 & 6{,}457 \\
& POI\_CAT.\_RECOG.            & 71.5  & 90.4  & $-$18.9 & SPATIAL\_LOCALIZATION       & 1.4 & 3{,}555 \\
& NAVIGATION                   & \textbf{62.5}  & 62.5  & +0.0  & ROUTE\_PLANNING             & 0.5 & 2{,}322 \\
& TRAJ.\_ANOMALY               & 28.6  & 67.4  & $-$38.8 & SPATIAL\_TRACKING           & 4.2 & 4{,}804 \\
& URBAN\_REG.\_FUNC.           & 6.3   & 47.6  & $-$41.3 & SPATIAL\_RELATIONSHIP       & 2.6 & 5{,}369 \\
& TRAJ.\_TRAJECTORY            & 27.9  & 44.5  & $-$16.6 & SPATIAL\_RELATIONSHIP        & 7.3 & 7{,}528 \\
& FLOW\_PREDICTION             & 12.0  & 12.0  & +0.1  & ST\_FORECAST                 & 1.2 & 3{,}418 \\
& GENERAL                      & 4.6   & 55.9  & $-$51.3 & SPATIAL\_RELATIONSHIP        & 3.6 & 6{,}201 \\
\midrule
\multicolumn{8}{l}{\textbf{STB26} (47.3\% cross vs.\ 48.9\% specific)} \\
& ETIOLOGICAL                  & \textbf{63.4}  & 62.2  & +1.2  & ST\_FORECAST                 & 6.0 & 10{,}731 \\
& CORRELATION                  & 51.8  & 54.8  & $-$3.1 & ST\_IMPUTE                   & 4.8 & 11{,}870 \\
& ENTITY                       & 33.7  & 38.9  & $-$5.3 & POI\_PREDICTION              & 4.4 & 9{,}888 \\
& FORECASTING                  & 10.3  & 10.3  & +0.0  & ST\_FORECAST                 & 1.7 & 7{,}350 \\
\bottomrule
\end{tabular}
\end{table*}

\FloatBarrier
\section{Cross-Model Per-Type Results}
\label{app:cross_model}

\begin{table*}[!htbp]
\centering
\caption{Cross-model EM (\%) on STARK per task type.  \textbf{Bold}: best per column.  $^\star$: numeric prediction column (SPATIOTEMPORAL\_FORECAST, TEMPORAL\_TRACKING); EM is not meaningful for these free-form real-valued outputs and is shown as ``--''.  Accuracy on those columns is reported in Table~\ref{tab:regression_metrics}.  Cell-level 95\% Wilson CIs are omitted in this dense layout; the per-benchmark totals in Table~\ref{tab:main} carry their own CIs.}
\label{tab:cross_stark}
\scriptsize
\setlength{\tabcolsep}{1.8pt}
\begin{tabular}{l|cccccccc|ccc|ccccccc|c}
\toprule
& \multicolumn{8}{c|}{Tier 1} & \multicolumn{3}{c|}{Tier 2} & \multicolumn{7}{c|}{Tier 3} & \\
\cmidrule(lr){2-9} \cmidrule(lr){10-12} \cmidrule(lr){13-19}
Model
 & \rotatebox{70}{SP\_IMP}
 & \rotatebox{70}{TM\_LOC}
 & \rotatebox{70}{TM\_IMP}
 & \rotatebox{70}{SP\_LOC}
 & \rotatebox{70}{ST\_IMP}
 & \rotatebox{70}{SP\_TRK}
 & \rotatebox{70}{ST\_FOR$^\star$}
 & \rotatebox{70}{TM\_TRK$^\star$}
 & \rotatebox{70}{TM\_REL}
 & \rotatebox{70}{SP\_REL}
 & \rotatebox{70}{ST\_REL}
 & \rotatebox{70}{LM\_DIR}
 & \rotatebox{70}{LM\_PRX}
 & \rotatebox{70}{ETA}
 & \rotatebox{70}{RT\_DUR}
 & \rotatebox{70}{POI}
 & \rotatebox{70}{INTENT}
 & \rotatebox{70}{RT\_PLN}
 & \rotatebox{70}{Total} \\
\midrule
Qwen3-8B       & \textbf{100} & \textbf{98} & \textbf{90} & 56 & 55 & \textbf{31} & --  & -- & \textbf{100} & \textbf{99} & \textbf{75} & \textbf{55} & 50 & 50 & 40 & 12 & 10 & 10 & \textbf{68.0} \\
Qwen3-VL-8B    & \textbf{100} & \textbf{98} & \textbf{90} & \textbf{57} & 55 & 23 & --  & -- & \textbf{100} & 91          & 58          & 50 & 48 & 65 & \textbf{55} & 6  & 6  & \textbf{40} & 61.6 \\
Qwen3-14B      & \textbf{100} & 97          & 10          & 57 & \textbf{60} & 30 & --  & -- & \textbf{100} & 97          & \textbf{78} & 40 & 50 & \textbf{77} & 50 & 3  & 0  & 5  & 53.0 \\
Gemma-3-12B    & \textbf{100} & 97          & 13          & 53 & \textbf{60} & 30 & --  & -- & \textbf{100} & 59          & 33          & 35 & 30 & \textbf{77} & 50 & \textbf{30} & \textbf{43} & 10 & 46.3 \\
Llama-3.1-8B   & 0            & 97          & 10          & 53 & 0  & 30 & --  & -- & \textbf{100} & 96          & 59          & 15 & 37 & 50 & 45 & \textbf{40} & 23 & 20 & 43.5 \\
Time-MQA-7B    & \textbf{100} & 97          & \textbf{90} & 57 & \textbf{60} & 27 & --  & -- & \textbf{100} & 40          & 19          & 50 & \textbf{50} & 46 & \textbf{65} & 33 & 20 & 10 & 46.3 \\
\bottomrule
\end{tabular}
\end{table*}

\begin{table*}[!htbp]
\centering
\caption{Cross-model EM (\%) on STB24 per task type.  \textbf{Bold}: best per column.  $^\star$: numeric prediction column (TRAJ\_PRD, FLOW); EM not meaningful, reported in Table~\ref{tab:regression_metrics}.  Per-benchmark totals carry 95\% Wilson CIs in Table~\ref{tab:main}.}
\label{tab:cross_stbench}
\scriptsize
\setlength{\tabcolsep}{1.8pt}
\begin{tabular}{l|ccccccccccccc|c}
\toprule
Model
 & \rotatebox{70}{DIR\_DET}
 & \rotatebox{70}{TRAJ\_PRD$^\star$}
 & \rotatebox{70}{PT\_TRAJ}
 & \rotatebox{70}{PT\_REG}
 & \rotatebox{70}{POI\_CAT}
 & \rotatebox{70}{TRAJ\_REG}
 & \rotatebox{70}{TRAJ\_ANO}
 & \rotatebox{70}{NAV}
 & \rotatebox{70}{URB\_FUN}
 & \rotatebox{70}{ADMIN}
 & \rotatebox{70}{TRAJ\_TRJ}
 & \rotatebox{70}{FLOW$^\star$}
 & \rotatebox{70}{GENERAL}
 & \rotatebox{70}{Total} \\
\midrule
Qwen3-8B       & \textbf{100} & --           & \textbf{100} & \textbf{93} & \textbf{90} & 81 & \textbf{67} & 63 & 46 & 42 & 27 & \textbf{12} & 10 & \textbf{50.0} \\
Qwen3-VL-8B    & \textbf{100} & --           & 97           & 92          & 88          & 82 & \textbf{67} & 67 & \textbf{52} & 36 & \textbf{35} & \textbf{12} & 8  & 51.1 \\
Qwen3-14B      & \textbf{100} & --           & \textbf{100} & \textbf{93} & 87          & \textbf{87} & 63 & \textbf{77} & 50 & 47 & \textbf{39} & 12 & \textbf{47} & \textbf{52.5} \\
Gemma-3-12B    & \textbf{100} & --           & 70           & 73          & 70          & \textbf{87} & 63 & \textbf{77} & 43 & \textbf{50} & 30 & 12 & 21 & 47.3 \\
Llama-3.1-8B   & 40           & --           & 33           & 73          & 67          & 3  & 63 & 53 & 17 & 23 & 20 & 12 & 5  & 25.2 \\
Time-MQA-7B    & \textbf{100} & --           & 13           & 43          & 3           & \textbf{87} & 60 & 27 & 17 & \textbf{47} & 1  & 12 & 0  & 27.1 \\
\bottomrule
\end{tabular}
\end{table*}

\begin{table}[!htbp]
\centering
\caption{Cross-model EM (\%) on STB26 per task type.  \textbf{Bold}: best per column.  $^\star$: numeric prediction column (FORECASTING); EM not meaningful, reported in Table~\ref{tab:regression_metrics}.  Per-benchmark totals carry 95\% Wilson CIs in Table~\ref{tab:main}.}
\label{tab:cross_stbenchnew}
\small
\begin{tabular}{lcccc|c}
\toprule
Model & CORR & ENTITY & ETIO & FORE$^\star$ & Total \\
\midrule
Qwen3-8B       & 52.8          & \textbf{40.0} & 56.1          & 10.3          & 44.3 \\
Qwen3-VL-8B    & \textbf{56.9} & 38.9          & \textbf{58.5} & 10.3          & \textbf{46.0} \\
Qwen3-14B      & 55.2          & 31.1          & 56.7          & \textbf{13.3} & 41.1 \\
Gemma-3-12B    & 24.1          & 8.9           & 23.3          & \textbf{13.3} & 17.8 \\
Llama-3.1-8B   & 51.7          & 35.6          & 50.0          & \textbf{13.3} & 39.9 \\
Time-MQA-7B    & 19.0          & 6.7           & 13.3          & 10.0          & 12.9 \\
\bottomrule
\end{tabular}
\end{table}

\FloatBarrier
\section{Statistical Significance}
\label{app:stat_sig}

All EM numbers reported in the main text and appendix are accompanied by 95\% Wilson-score confidence intervals on the underlying binomial outcomes.  EM is, by construction, a binary 0/1 indicator per query (after the answer-normalization pipeline of Round~14c, including binary yes/no mapping and numeric-prediction exclusion), so the per-task or per-benchmark EM is a binomial proportion and admits a closed-form Wilson interval:
\[
\text{CI}_{95}(p) =
\frac{p + \tfrac{z^{2}}{2n}}{1 + \tfrac{z^{2}}{n}}
\;\pm\;
\frac{z}{1 + \tfrac{z^{2}}{n}} \sqrt{\frac{p(1-p)}{n} + \frac{z^{2}}{4 n^{2}}},
\qquad
z = 1.96,
\]
where $p = k/n$ is the observed EM proportion, $k$ is the number of correct queries, and $n$ is the number of EM-eligible queries contributing to the cell.  The tabulated ``$\pm$'' value is the half-width of this interval.

The Wilson interval is preferred to the naive Wald interval because it has better coverage near $p=0$ and $p=1$ (of which we have several cases, e.g., STB24 \texttt{DIRECTION\_DET.} at EM$=100\%$), and because it is a closed-form function of $(k, n)$ directly available from each result file without additional evaluation.

\paragraph{Which cells carry CIs.}  Headline tables (Tables~\ref{tab:main}, \ref{tab:baselines}, \ref{tab:framework_ablation}) and per-task-type tables for Qwen3-8B (Tables~\ref{tab:stark_full}, \ref{tab:stbench_full}, \ref{tab:stbenchnew_full}) report the CI half-width inline after each EM value.  The cross-model per-type tables (Tables~\ref{tab:cross_stark}, \ref{tab:cross_stbench}, \ref{tab:cross_stbenchnew}) omit per-cell CIs for layout reasons; the per-benchmark totals in Table~\ref{tab:main} carry their own CIs and serve as the aggregate inference target for those comparisons.  Regression metrics (RMSE, RMSPE, MAE, AbsErr) use dataset-level point estimates; the five numeric prediction task types for which EM is not meaningful (STARK SPATIOTEMPORAL\_FORECAST and TEMPORAL\_TRACKING, STB24 FLOW\_PREDICTION and TRAJECTORY\_PREDICTION, STB26 FORECASTING) are excluded from the EM averages and marked with $^\star$ in all per-type tables.

\paragraph{Independence.}  Queries within a benchmark are treated as independent binomial trials for the purpose of the CI; task-type clustering is already reflected in the per-task cell sizes, so per-cell CIs do not require correction.  When an ablation or baseline changes only a subset of queries, we additionally report the ``on diverging queries'' restriction of Table~\ref{tab:router_ablation}, which is the direct paired comparison against the full system and therefore the most sensitive measure of routing-specific effects.

\paragraph{Effective sample sizes.}  EM-eligible $n$ per cell: STARK $n=2{,}792$, STB24 $n=2{,}443$, STB26 $n=1{,}128$ for Qwen3-8B and other fully-evaluated open-source backbones.

\FloatBarrier
\section{Agent and Computation Specifications}
\label{app:agents}

\begin{table}[!htbp]
\centering
\caption{Agent computation menus $\mathcal{M}_a$.  Each agent exposes a declarative list of deterministic functions; the LLM selects which to invoke and extracts parameters.  Implementation details: spatial computations use Shapely and NumPy; temporal computations implement Allen's Interval Algebra; graph computations use NetworkX with custom causality modules; navigation uses Dijkstra's algorithm.}
\label{tab:tools}
\small
\begin{tabular}{lp{2.2cm}p{5.8cm}}
\toprule
Agent & LLM Role & Computation Menu $\mathcal{M}_a$ \\
\midrule
\textsc{Head} & Classify task type & --- (classification and profile extraction) \\
\textsc{Spatial} & Select computation, extract parameters & Spatial relation verification (7 relation types), distance computation, direction computation, trajectory-region membership, localization (bearing/range triangulation) \\
\textsc{Temporal} & Extract time intervals & 13 interval relations, forecasting (seasonal + moving average ensemble), interval set operations \\
\textsc{Trajectory} & Extract trajectory data & Anomaly detection, prediction, region classification \\
\textsc{Topological} & Parse graph structure & Topology analysis, centrality metrics, cascade detection, pairwise causality analysis, feature vector extraction, option verification \\
\textsc{Navigation} & Extract origin/dest. & Shortest path computation, ETA estimation \\
\textsc{Semantic} & Full understanding & --- (knowledge-based QA, no deterministic computation) \\
\textsc{Fusion} & Interpret evidence & --- (synthesis from blackboard $\Board_K$) \\
\bottomrule
\end{tabular}
\end{table}

\FloatBarrier
\section{Transition Matrix Structure}
\label{app:matrix}

The learned transition matrix $\mathbf{M}$ exhibits several structural properties:

\begin{itemize}
    \item \textbf{Sparsity}: Most task types activate only 2--3 agents in sequence, yielding sparse row distributions.  The average number of non-zero entries per row (conditioned on task type) is 1.8, reflecting focused routing.

    \item \textbf{Absorbing fusion}: The terminal constraint $\mathbf{M}[a_{\sfuse}, \cdot, \cdot, a'] = 0$ for $a' \neq a_{\sfuse}$ is enforced structurally, ensuring routing always converges.

    \item \textbf{Task-type specificity}: Different task types produce markedly different transition structures.  Spatiotemporal relationship tasks route through both spatial and temporal agents in sequence, while distance tasks route only through the spatial agent.

    \item \textbf{Status-conditioned recovery}: Error-state rows ($s \in \{\sfail, \sblock, \smiss\}$) have focused support on 1--2 alternative agents, indicating that the training data provides targeted recovery signals.  The \smiss status consistently routes to different recovery agents than \sfail, reflecting the distinct semantics of missing data versus computation errors.
\end{itemize}

Figures~\ref{fig:matrix_heatmaps} and~\ref{fig:matrix_topk} visualize these properties directly.
For each benchmark we aggregate $\mathbf{M}[a, s, t, a']$ across task types and plot the resulting $(\textit{from\_agent}, \textit{to\_agent})$ probability maps separately for each status $s \in \{\ssucc, \sfail, \smiss, \sblock\}$.
The heatmaps show \emph{what} the matrix routes to, while the top-$k$ bar charts show the three failure statuses side-by-side to expose the first-order claim: \sfail, \smiss, and \sblock route to \emph{different} successors from the same originating agent, which is the empirical counterpart of the status-conditioning argument in \S\ref{sec:formulation}.

\begin{figure}[p]
\centering
\begin{subfigure}[b]{\textwidth}
    \centering
    \includegraphics[width=\linewidth]{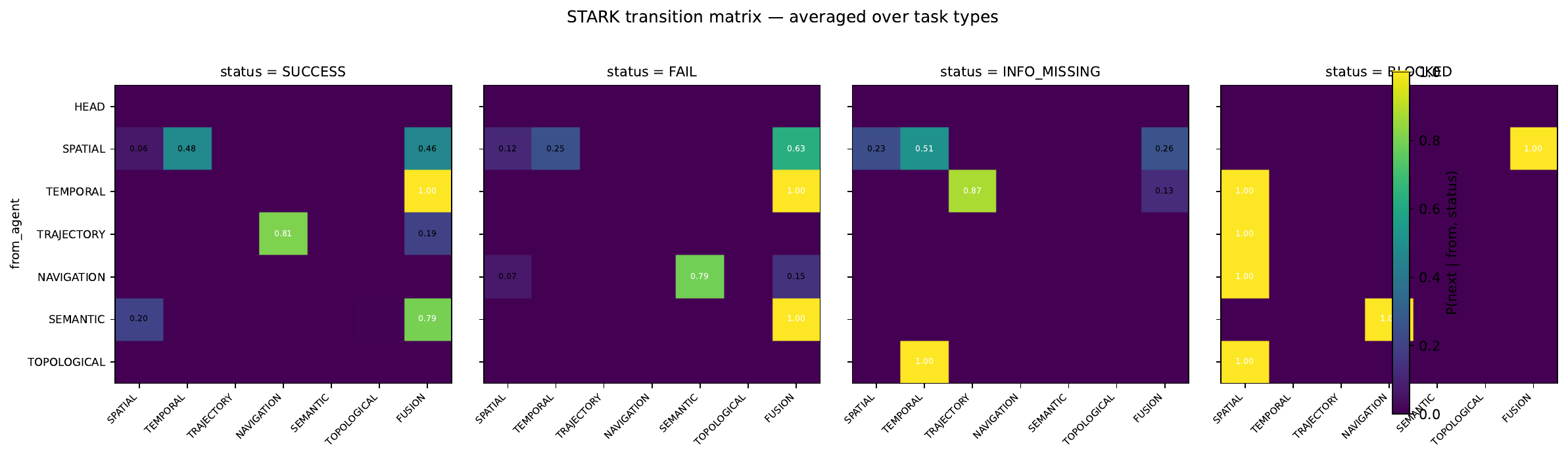}
    \caption{STARK.}
    \label{fig:heatmap_stark}
\end{subfigure}

\vspace{0.35em}

\begin{subfigure}[b]{\textwidth}
    \centering
    \includegraphics[width=\linewidth]{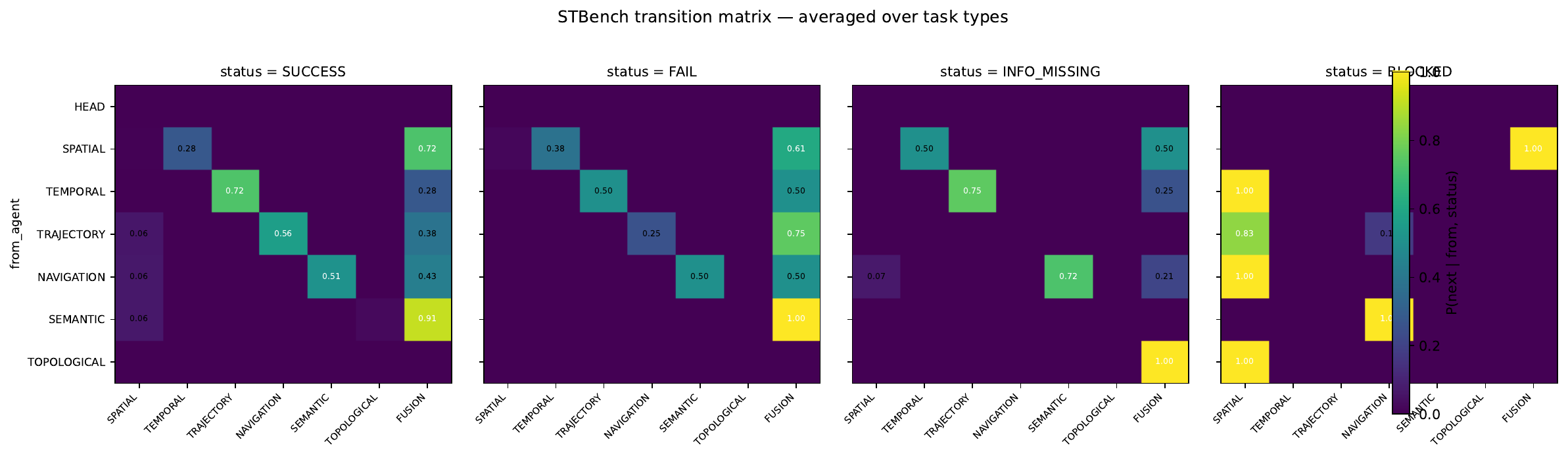}
    \caption{STB24.}
    \label{fig:heatmap_stbench}
\end{subfigure}

\vspace{0.35em}

\begin{subfigure}[b]{\textwidth}
    \centering
    \includegraphics[width=\linewidth]{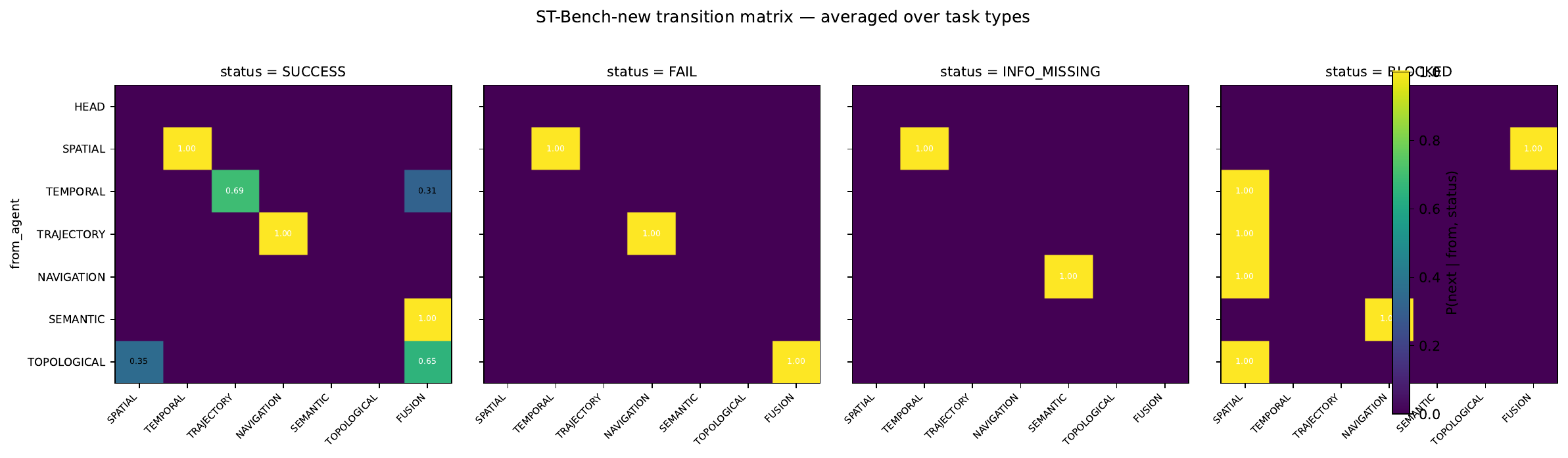}
    \caption{STB26.}
    \label{fig:heatmap_stbenchnew}
\end{subfigure}
\caption{Learned transition matrix $\mathbf{M}$ averaged across task types, rendered as four status-conditional heatmaps (SUCCESS / FAIL / INFO\_MISSING / BLOCKED).  Rows index the \emph{from\_agent}, columns index the \emph{to\_agent}, and each cell is $P(\textit{next}\mid \textit{from}, \textit{status})$.  Each benchmark produces a qualitatively different SUCCESS map (because task taxonomies differ), but every benchmark exhibits the same structural property---error-state rows have non-zero support concentrated on different successors from the SUCCESS row, which is only possible because unsuccessful traces are retained during training (Theorem~\ref{thm:recovery}).  BLOCKED transitions are populated by the hand-curated safety-net table in System~1 plus any task-specific BLOCKED rules present in the matrix.}
\label{fig:matrix_heatmaps}
\end{figure}

\begin{figure}[p]
\centering
\begin{subfigure}[b]{\textwidth}
    \centering
    \includegraphics[width=0.72\linewidth]{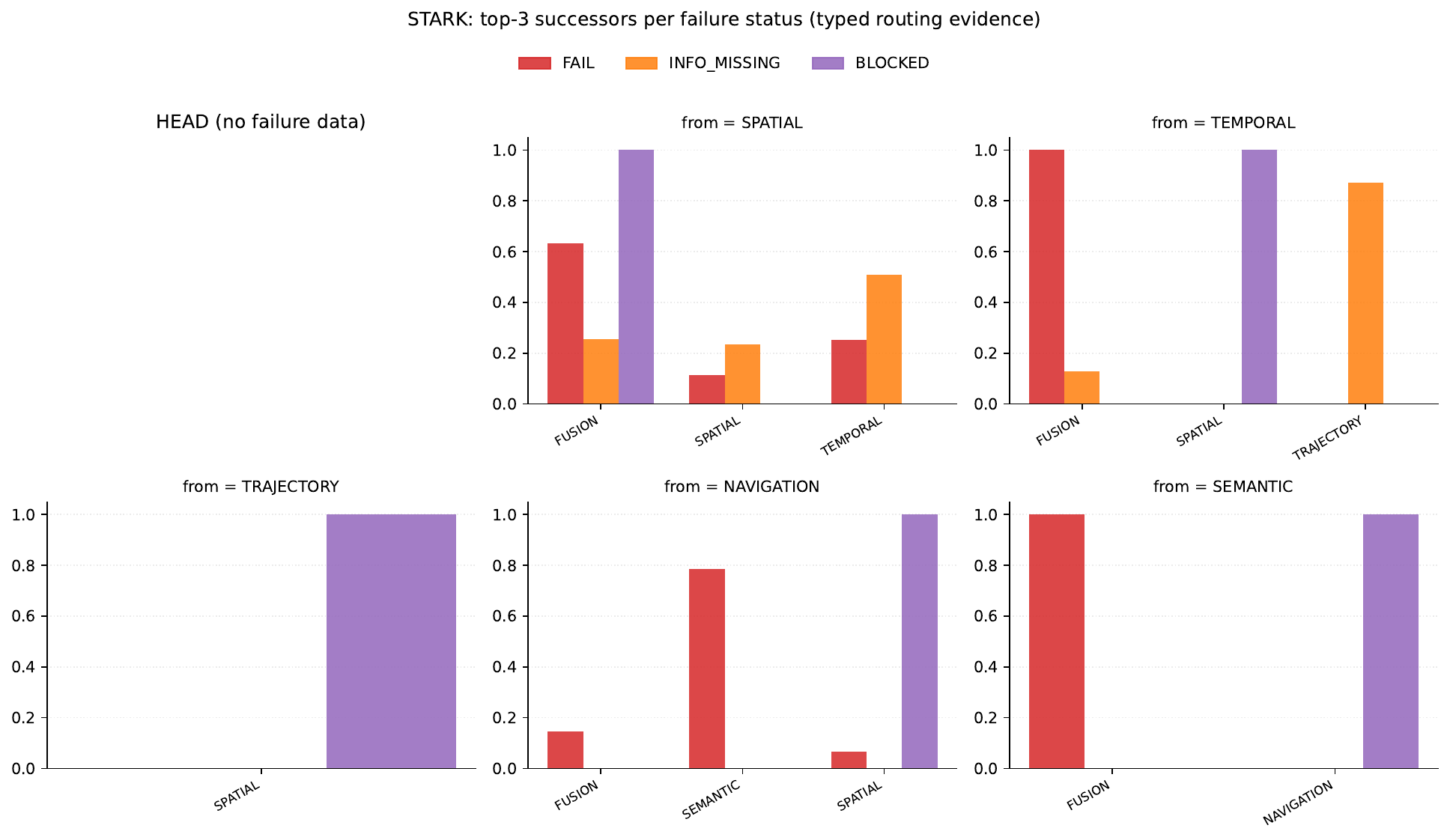}
    \caption{STARK.}
    \label{fig:topk_stark}
\end{subfigure}

\vspace{0.35em}

\begin{subfigure}[b]{\textwidth}
    \centering
    \includegraphics[width=0.72\linewidth]{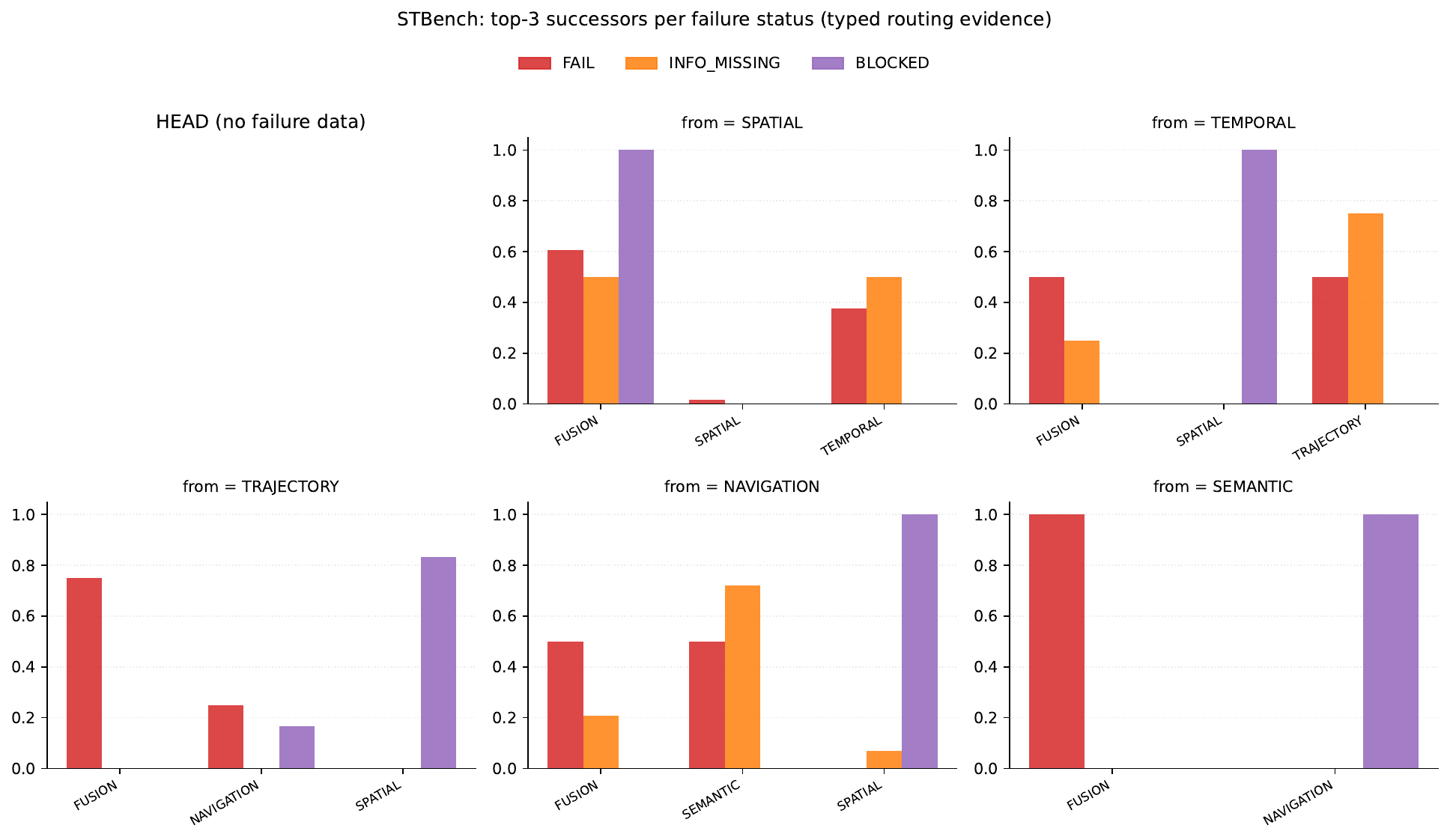}
    \caption{STB24.}
    \label{fig:topk_stbench}
\end{subfigure}

\vspace{0.35em}

\begin{subfigure}[b]{\textwidth}
    \centering
    \includegraphics[width=0.72\linewidth]{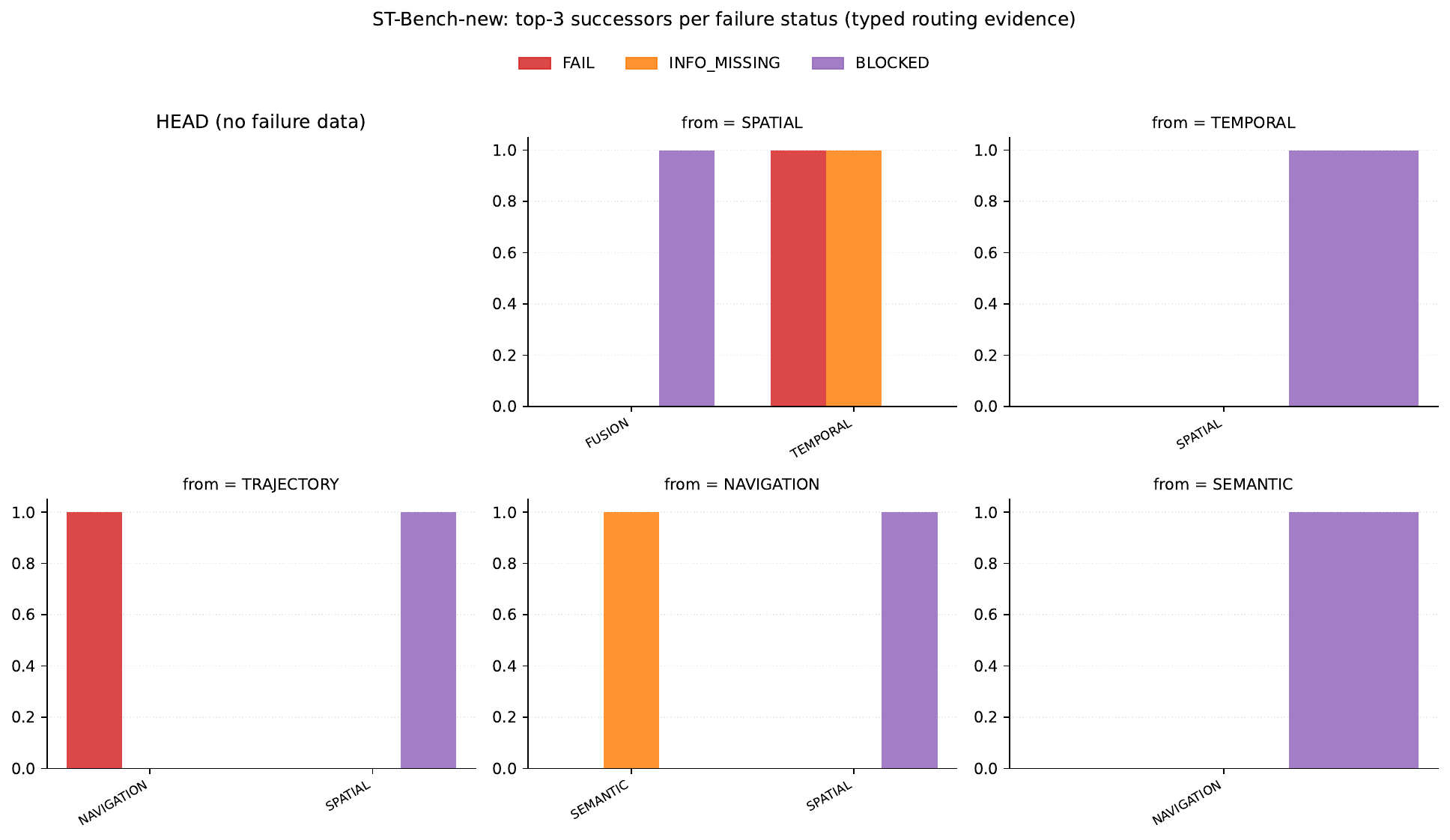}
    \caption{STB26.}
    \label{fig:topk_stbenchnew}
\end{subfigure}
\caption{Top-3 successors per failure status, grouped by originating agent.  For every \emph{from\_agent} in each benchmark, the three bars per successor show $P(\textit{next}\mid \textit{from}, \textit{status})$ for $\textit{status} \in \{\sfail, \smiss, \sblock\}$.  The claim visualized here is that the \emph{same} $(\textit{from\_agent}, \textit{task\_type})$ routes to \emph{different} recovery agents depending on the observed error type---the height pattern shifts across the three colored bars within each group.  This is the direct empirical counterpart of the status-conditioning ablation in Table~\ref{tab:router_ablation} (row ``No status conditioning''), which collapsed these three bars into one and paid up to 33pp on diverging STARK queries.}
\label{fig:matrix_topk}
\end{figure}

\FloatBarrier
\section{Composite Query Case Studies}
\label{app:composite}

Two composite query traces showing the full pipeline: HEAD decomposition $\to$ parallel agent execution $\to$ blackboard merge $\to$ multi-part FUSION.

\begin{table*}[!htbp]
\centering
\caption{\textbf{Composite Case 1}: Insurance claim verification (Event Analysis, SPATIAL+TEMPORAL).  Query: \textit{Is the vehicle trajectory within the parking zone? Does the claimed time window [1.5, 4.5] fall during the active window [0.0, 6.0]?}  GT: spatial=1, temporal=1.}
\label{tab:comp_case1}
\small
\begin{tabular}{c l p{9.5cm}}
\toprule
Step & Component & Trace \\
\midrule
1 & HEAD &
\texttt{LLM:} \texttt{\{"sub\_types": ["SPATIAL", "TEMPORAL"], "reasoning": "spatial analysis + temporal analysis"\}} \\
\midrule
2a & SPATIAL \newline (parallel) &
\texttt{LLM:} \texttt{<JSON>\{"operation": "within", "geom\_1": [[2.5,3.0],\ldots,[5.0,5.0]], "geom\_1\_type": "linestring", "geom\_2": [[2.0,2.0],\ldots,[2.0,2.0]], "geom\_2\_type": "polygon", "compute\_event\_interval": true\}</JSON>} \newline
\textbf{Tool}: \texttt{check\_within} $\to$ \texttt{True}. Event interval: [1.0, 5.0]. \\
\midrule
2b & TEMPORAL \newline (parallel) &
\texttt{LLM:} \texttt{<JSON>\{"relation": "overlaps", "interval\_1": [1.5, 4.5], "interval\_2": [0.0, 6.0]\}</JSON>} \newline
\textbf{Tool}: Allen's \texttt{during([1.5,4.5], [0.0,6.0])} $\to$ \texttt{True}. \\
\midrule
& $\Board$ &
\texttt{spatial\_data: \{within: True, event\_interval: [1.0, 5.0]\}} \newline
\texttt{temporal\_data: \{relation\_holds: True\}} \\
\midrule
3 & FUSION &
Reads both results $\to$ \texttt{\{"part1": 1, "part2": 1\}} \\
\midrule
& \textbf{Route} & HEAD $\to$ [SPATIAL $\|$ TEMPORAL] $\to$ FUSION \hfill \textbf{Both parts correct} $\checkmark$ \\
\bottomrule
\end{tabular}
\end{table*}

\begin{table*}[!htbp]
\centering
\caption{\textbf{Composite Case 2}: Weather monitoring (Event Analysis, SPATIAL+TEMPORAL).  Query: \textit{Does the storm trajectory intersect the sensor area? Does the storm's time [1.0, 3.0] overlap with maintenance [1.0, 2.5]?}  GT: spatial=1, temporal=1.}
\label{tab:comp_case2}
\small
\begin{tabular}{c l p{9.5cm}}
\toprule
Step & Component & Trace \\
\midrule
1 & HEAD &
\texttt{LLM:} \texttt{\{"sub\_types": ["SPATIAL", "TEMPORAL"], "reasoning": "spatial intersection + temporal overlap analysis"\}} \\
\midrule
2a & SPATIAL \newline (parallel) &
\texttt{LLM:} \texttt{<JSON>\{"operation": "intersects", "geom\_1": [[1.0,8.0],\ldots,[5.0,4.0]], "geom\_1\_type": "linestring", "geom\_2": [[2.5,4.5],\ldots,[2.5,4.5]], "geom\_2\_type": "polygon"\}</JSON>} \newline
\textbf{Tool}: \texttt{check\_intersects} $\to$ \texttt{True}. \\
\midrule
2b & TEMPORAL \newline (parallel) &
\texttt{LLM:} \texttt{<JSON>\{"relation": "overlaps", "interval\_1": [1.0, 3.0], "interval\_2": [1.0, 2.5]\}</JSON>} \newline
\textbf{Tool}: Allen's \texttt{overlaps([1.0,3.0], [1.0,2.5])} $\to$ \texttt{False} (starts equal, not strictly overlaps). \\
\midrule
& $\Board$ &
\texttt{spatial\_data: \{intersects: True\}}, \texttt{temporal\_data: \{overlaps: False\}} \\
\midrule
3 & FUSION &
Reads both $\to$ \texttt{\{"part1": 1, "part2": 0\}} \hfill (part2 incorrect: GT=1) \\
\midrule
& \textbf{Route} & HEAD $\to$ [SPATIAL $\|$ TEMPORAL] $\to$ FUSION \hfill \textbf{Partial: 1/2 correct} \\
\bottomrule
\end{tabular}
\end{table*}

\paragraph{Analysis.}
Case~1 demonstrates perfect multi-agent collaboration: HEAD decomposes into [SPATIAL, TEMPORAL], both agents execute in parallel, and FUSION synthesizes both parts correctly.
Case~2 shows an interesting error mode: the TEMPORAL agent selects Allen's \texttt{overlaps} but the intervals [1.0, 3.0] and [1.0, 2.5] share a start point, making the strict Allen relation \texttt{False}---the correct relation is \texttt{starts} (one interval starts at the same point and extends beyond).
This is a parameter extraction error (the LLM selected the wrong Allen relation), not a tool or routing error.

\paragraph{Agent activation patterns.}
Across all 14 composite queries, the HEAD agent correctly identifies the needed agent types in the majority of cases.
Trajectory monitoring queries (4/4) achieve correct decomposition into [TRAJECTORY, SPATIAL] with both agents returning $\ssucc$, yielding 75\% overall accuracy.
Event analysis achieves 70\% partial accuracy---spatial sub-parts are always correct, while temporal sub-parts occasionally fail due to Allen relation selection errors.
Logistics queries are limited by the navigation agent's ability to parse natural-text graph descriptions.

\FloatBarrier
\section{Case Studies with Full Traces}
\label{app:cases}

This section presents five worked examples drawn directly from benchmark test splits and the corresponding \starname{} execution traces recorded during the Qwen3-8B evaluation.  Each case reports:
(i) the raw query id and source file,
(ii) an excerpt of the query text as it enters the \shead{} agent (verbatim apart from whitespace cleanup),
(iii) the ground-truth answer,
(iv) the step-by-step trace including each agent's parameter-extraction JSON, the deterministic tool call, the blackboard update, and the final \sfuse{} synthesis.
Prompts are shortened with ``\texttt{[\ldots]}'' where truncated; JSON fields and numeric values are reproduced from the actual trace files.
The five cases are selected to cover (A)~a single-tool success, (B)~multi-agent blackboard composition, (C)~external-API tool use, (D)~\smiss-driven System~2 recovery, and (E)~graph-structured reasoning on the STB26 benchmark.

\subsection{Case A --- Single-tool success (STB24, Direction Determination)}
\label{app:case_a}

\textbf{Source.} \texttt{data/splits/stbench\_test.jsonl}, id \texttt{\seqsplit{stbench\_basic\_direction\_determination\_2261}}.
\textbf{Task type.} \texttt{STBENCH\_DIRECTION\_DETERMINATION}.
\textbf{Ground truth.} \texttt{1}.

\textbf{Raw query (verbatim).}
\begin{quote}\small\ttfamily
Question: A has a longitude of 120.1204 and a latitude of 30.8661, while B has a longitude of 128.3270 and a latitude of 45.458311. Therefore, B is in the () from A. Please choose the correct answer from the following options and fill it in parentheses. (1) North, (2) Northeast, (3) East, (4) Southeast, (5) South, (6) Southwest, (7) West, (8) Northwest. Please directly give me the number of your option with no other texts. Answer: Option (
\end{quote}

\begin{table}[H]
\centering
\caption{\textbf{Case~A}: execution trace.  Route: $\shead \xrightarrow{\ssucc} \ssp \xrightarrow{\ssucc} \sfuse$; predicted \texttt{1} $=$ GT.}
\label{tab:case_a}
\small
\begin{tabular}{c l p{10.5cm}}
\toprule
Step & Component & Trace \\
\midrule
1 & \shead{} &
\textit{System prompt.} Classify the query into exactly one \texttt{task\_type} from the 26-element catalog and emit \texttt{<JSON>\ldots</JSON>}. \newline
\textit{Model output.} \texttt{<JSON>\{"task\_type": "\seqsplit{STBENCH\_DIRECTION\_DETERMINATION}", "constraints": ["compass direction", "point coordinates"], "benchmark": "STBench"\}</JSON>} \newline
System~1 lookup for this \texttt{task\_type} $\Rightarrow$ primary agent $=$ \ssp{}. \\
\midrule
2 & \ssp{} &
\textit{Extraction prompt.} ``Pick one operation from: contains, crosses, intersects, within, touches, overlaps, equals, compass\_direction, haversine\_distance, landmark\_direction, admin\_region\_lookup, compute\_event\_interval, localization''. \newline
\textit{Model output.} \texttt{<JSON>\{"operation": "compass\_direction", "geom\_1": [120.1204, 30.8661], "geom\_2": [128.3270, 45.458311], "coordinates\_resolved": true\}</JSON>} \\
\midrule
  & Tool &
\texttt{tools.spatial\_ops.compute\_compass\_direction((120.1204, 30.8661), (128.3270, 45.458311))} returns bearing $\approx 27.3^{\circ}$ (clockwise from north) $\mapsto$ wedge (1)~North. \\
\midrule
  & $\Board$ update &
\texttt{spatial\_data = \{"direction\_result": 1, "bearing\_deg": 27.3\}} \\
\midrule
3 & \sfuse{} &
Reads \texttt{spatial\_data.direction\_result}, formats per task-type template $\Rightarrow$ \texttt{<answer>1</answer>}. \\
\bottomrule
\end{tabular}
\end{table}
\FloatBarrier
\subsection{Case B --- Multi-agent composition via blackboard (STARK, Spatiotemporal Relationship)}
\label{app:case_b}

\textbf{Source.} \texttt{data/splits/stark\_test.jsonl}, id \texttt{\seqsplit{stark\_Linestring\_Polygon\_within-during\_27}}.
\textbf{Task type.} \texttt{STARK\_SPATIOTEMPORAL\_RELATIONSHIP}.
\textbf{Ground truth.} \texttt{[0.0]} (relation does \emph{not} hold).

\textbf{Raw query (excerpt, non-boilerplate portion).}
\begin{quote}\small\ttfamily
Objective: Determine whether the time interval during which the EVENT holds has the temporal relationship \textbf{during} with the reference interval (6.5003, 11.8556)?

EVENT: the following object trajectory has the spatial relationship \textbf{within} with Polygon [(3.1253, 9.5958), (0.3137, 11.2892), (-2.1656, 9.1386), (-0.8864, 6.1160), (2.3836, 6.3986), (3.1253, 9.5958)]

Object trajectory: [(0.0000, 10.0000), (0.0826, 9.7639), (0.1701, 9.4990), (0.2650, 9.2067), (0.3698, 8.8887), (0.4865, 8.5477), (0.6172, 8.1864), (0.7636, 7.8081), (0.9269, 7.4164), (1.1082, 7.0153)]

Timestamp: [1.5770, 2.5974, 3.6178, 4.6382, 5.6586, 6.6790, 7.6994, 8.7198, 9.7402, 10.7606]

Answer 1 if Yes. Otherwise, answer 0.
[RESULTS\_START] [p] [RESULTS\_END]
\end{quote}

\begin{table}[H]
\centering
\caption{\textbf{Case~B}: multi-agent execution trace.  Route: $\shead \xrightarrow{\ssucc} \ssp \xrightarrow{\ssucc} \stp \xrightarrow{\ssucc} \sfuse$; predicted \texttt{[0.0]} $=$ GT.  The \ssp{} agent produces an intermediate \emph{event interval} that the \stp{} agent consumes from the blackboard --- without this composition the question is unanswerable by either agent alone.}
\label{tab:case_b}
\small
\begin{tabular}{c l p{10.5cm}}
\toprule
Step & Component & Trace \\
\midrule
1 & \shead{} &
\textit{Model output.} \texttt{<JSON>\{"task\_type": "\seqsplit{STARK\_SPATIOTEMPORAL\_RELATIONSHIP}", "constraints": ["spatial relation: within", "event interval", "temporal relation: during"], "benchmark": "STARK"\}</JSON>} \newline
System~1 $\Rightarrow$ primary $=$ \ssp{} (spatial predicate must be evaluated first so its event interval can feed \stp{}). \\
\midrule
2 & \ssp{} &
\textit{Extraction.} The agent must (i)~evaluate \texttt{within} pointwise, (ii)~project satisfying segments back to their timestamps, (iii)~emit the event interval. \newline
\textit{Model output.} \texttt{\seqsplit{<JSON>\{"operation": "within", "geom\_1": [[0.0,10.0],[0.083,9.76],[0.17,9.50],[0.27,9.21],[0.37,8.89],[0.49,8.55],[0.62,8.19],[0.76,7.81],[0.93,7.42],[1.11,7.02]], "geom\_1\_type": "linestring", "geom\_2": [[3.13,9.60],[0.31,11.29],[-2.17,9.14],[-0.89,6.12],[2.38,6.40],[3.13,9.60]], "geom\_2\_type": "polygon", "timestamps": [1.577, 2.597, 3.618, 4.638, 5.659, 6.679, 7.699, 8.720, 9.740, 10.761], "compute\_event\_interval": true\}</JSON>}} \\
\midrule
  & Tool &
\texttt{tools.spatial\_ops.within\_with\_event\_interval(traj, poly, timestamps)}.  Per-vertex Shapely within tests are all \texttt{True}; projecting first and last satisfying indices back yields event interval \texttt{[1.577, 10.761]}. \\
\midrule
  & $\Board$ update &
\texttt{spatial\_data = \{"spatial\_relation\_holds": True, "event\_interval": [1.577, 10.761]\}} \\
\midrule
3 & \stp{} &
\textit{Extraction.}  Reads \texttt{spatial\_data.event\_interval} from $\Board$; emits the Allen predicate. \newline
\textit{Model output.} \texttt{<JSON>\{"operation": "allen\_during", "interval\_a": [1.577, 10.761], "interval\_b": [6.5003, 11.8556]\}</JSON>} \\
\midrule
  & Tool &
\texttt{tools.temporal\_ops.allen\_during([1.577,10.761], [6.5003,11.8556])}: \emph{during}(A,B) requires $a_{\text{start}} > b_{\text{start}}$ and $a_{\text{end}} < b_{\text{end}}$. Here $1.577 < 6.5003$, so the relation is \texttt{False}. \\
\midrule
  & $\Board$ update &
\texttt{temporal\_data = \{"allen\_relation": "during", "holds": False\}} \\
\midrule
4 & \sfuse{} &
Reads both \texttt{spatial\_data.spatial\_relation\_holds}=True and \texttt{temporal\_data.holds}=False.  The task template requires both to hold for the combined relation $\Rightarrow$ answer $0$. Emits \texttt{[RESULTS\_START] [0.0] [RESULTS\_END]}. \\
\bottomrule
\end{tabular}
\end{table}
\FloatBarrier
\subsection{Case C --- External-API geocoding tool (STB24, Administrative Region)}
\label{app:case_c}

\textbf{Source.} \texttt{data/splits/stbench\_test.jsonl}, id \texttt{\seqsplit{stbench\_basic\_administrative\_region\_determination\_3977}}.
\textbf{Task type.} \texttt{STBENCH\_ADMIN\_REGION}.
\textbf{Ground truth.} \texttt{4}.

\textbf{Raw query (verbatim).}
\begin{quote}\small\ttfamily
Question: Below is the coordinate location information, and the options of the area where the coordinate may be located:
\{"latitude": 39.9407892, "longitude": -75.2596751,
 "options": "(0): Newport, PA (1): Presto, PA (2): Sturgeon, PA (3): Fogelsville, PA (4): Lansdowne, PA"\}
Please answer which area the coordinate is located in. Please just answer the number of your option with no other texts.
Answer: Option (
\end{quote}

\begin{table}[H]
\centering
\caption{\textbf{Case~C}: external-API tool trace.  Route: $\shead \xrightarrow{\ssucc} \ssp \xrightarrow{\ssucc} \sfuse$; predicted \texttt{4} $=$ GT.}
\label{tab:case_c}
\small
\begin{tabular}{c l p{10.5cm}}
\toprule
Step & Component & Trace \\
\midrule
1 & \shead{} &
\textit{Model output.} \texttt{<JSON>\{"task\_type": "STBENCH\_ADMIN\_REGION", "constraints": ["lat/lon coordinate", "5 candidate PA municipalities"], "benchmark": "STBench"\}</JSON>} \\
\midrule
2 & \ssp{} &
\textit{Extraction.}  \texttt{<JSON>\{"operation": "admin\_region\_lookup", "coordinates": [\{"latitude": 39.9407892, "longitude": -75.2596751\}], "options": ["Newport, PA", "Presto, PA", "Sturgeon, PA", "Fogelsville, PA", "Lansdowne, PA"], "coordinates\_resolved": true\}</JSON>} \\
\midrule
  & Tool &
\texttt{tools.spatial\_ops.admin\_region\_lookup(39.9408, -75.2597)}: reverse-geocodes via Nominatim, returns \texttt{"East Lansdowne, Delaware County, Pennsylvania, USA"}.  A rapidfuzz partial-ratio match against the five options scores ``Lansdowne, PA'' at 0.91, others $< 0.3$ $\Rightarrow$ option \texttt{4}. \\
\midrule
  & $\Board$ update &
\texttt{spatial\_data = \{"geocoded\_name": "East Lansdowne, Delaware County, Pennsylvania, USA", "matched\_option": 4, "match\_score": 0.91\}} \\
\midrule
3 & \sfuse{} &
Reads \texttt{matched\_option}, emits \texttt{<answer>4</answer>}. \\
\bottomrule
\end{tabular}
\end{table}
\FloatBarrier
\subsection{Case D --- \smiss{}-driven System~2 recovery (STARK, Landmark Direction)}
\label{app:case_d}

\textbf{Source.} \texttt{data/splits/stark\_test.jsonl}, id \texttt{stark\_direction\_questions\_37}.
\textbf{Task type.} \texttt{STARK\_LANDMARK\_DIRECTION}.
\textbf{Ground truth.} \texttt{[1.0]} (the proposed direction holds).

\textbf{Raw query (non-boilerplate portion).}
\begin{quote}\small\ttfamily
Objective: Determine the most accurate spatial relationship between Alcatraz Island, San Francisco, CA and Union Square, San Francisco, CA is south of, selecting from the options: ['north of', 'south of', 'east of', 'west of', 'north-east of', 'north-west of', 'south-east of', 'south-west of'].  The first location (A) is the reference location.  The direction tells you where B is with respect to that reference.  Answer 1 if answer is Yes. Otherwise, answer 0.
[RESULTS\_START] [p] [RESULTS\_END]
\end{quote}

\begin{table}[H]
\centering
\caption{\textbf{Case~D}: recovery via \smiss.  Route: $\shead \xrightarrow{\ssucc} \ssp \xrightarrow{\smiss} \sfuse_{\text{Sys.~2}}$; predicted \texttt{[1.0]} $=$ GT.  The primary \ssp{} call cannot resolve named landmarks without a gazetteer of SF points of interest, so it emits \smiss{} rather than fabricating coordinates.  The System~2 row of the transition matrix routes $(\ssp, \smiss, \texttt{STARK\_LANDMARK\_DIRECTION}) \mapsto \sfuse$ because the query text alone is sufficient for a world-knowledge answer.}
\label{tab:case_d}
\small
\begin{tabular}{c l p{10.5cm}}
\toprule
Step & Component & Trace \\
\midrule
1 & \shead{} &
\texttt{<JSON>\{"task\_type": "STARK\_LANDMARK\_DIRECTION", "constraints": ["named landmark pair", "compass wedge"], "benchmark": "STARK"\}</JSON>} \\
\midrule
2 & \ssp{} &
\textit{Extraction.}  \texttt{<JSON>\{"operation": "landmark\_direction", "geom\_1": "Alcatraz Island, San Francisco, CA", "geom\_2": "Union Square, San Francisco, CA", "proposed\_direction": "south of", "coordinates\_resolved": false\}</JSON>} \newline
Because \texttt{coordinates\_resolved=false} and no \texttt{navigation\_data.poi\_registry} entry is available on $\Board$ for either name, the agent returns \texttt{AgentResult(INFO\_MISSING, \{"missing": "landmark\_coordinates", "raw\_names": ["Alcatraz Island, San Francisco, CA", "Union Square, San Francisco, CA"]\})}. \\
\midrule
  & $\Board$ update &
\texttt{spatial\_data = \{"status": "INFO\_MISSING", "raw\_query\_echo": "Is Union Square south of Alcatraz Island?"\}} \\
\midrule
3 & Router &
System~1 exhausted for this \texttt{task\_type}; System~2 looks up $\mathbf{M}[\texttt{STARK\_LANDMARK\_DIRECTION}, \ssp, \smiss, \cdot]$.  The largest-mass successor is \sfuse{} (trained on traces where the LLM answered from background knowledge after tool failure).  The router dispatches \sfuse{} directly. \\
\midrule
4 & \sfuse{} &
Consumes the raw query plus the ``tool could not resolve landmarks'' flag.  Uses world knowledge that Union Square ($37.788^{\circ}$N) lies south of Alcatraz Island ($37.827^{\circ}$N), so ``south of'' is correct. \newline
\textit{Model output.} \texttt{[RESULTS\_START] [1.0] [RESULTS\_END]}. \\
\bottomrule
\end{tabular}
\end{table}
\FloatBarrier
\subsection{Case E --- Graph-structured etiological reasoning (STB26, Etiological MCQ)}
\label{app:case_e}

\textbf{Source.} \texttt{data/splits/st\_bench\_new\_test.jsonl}, id \texttt{st\_bench\_new\_5305}.
\textbf{Task type.} \texttt{ST\_BENCH\_NEW\_ETIOLOGICAL}.
\textbf{Ground truth.} \texttt{<answer>C</answer>}.

\textbf{Raw query (excerpt; all five 48-step time series are provided in full in the original prompt).}
\begin{quote}\small\ttfamily
You are a spatial temporal analysis expert.

Node 0 time series with length of 48: [300.7, 304.55, 290.54, 276.78, 272.19, 286.31, 291.35, 285.58, 302.58, 324.5, 321.06, \ldots, 271.54, 263.42];

Node 1 time series with length of 48: [348.36, 331.46, 327.99, 327.03, 315.48, 313.43, 307.54, 302.02, 307.6, 313.39, \ldots, 317.57, 305.87];

Node 2 \ldots{} [246.92, \ldots, 277.26];

Node 3 \ldots{} [275.33, \ldots, 287.06];

Node 4 \ldots{} [222.21, \ldots, 262.21];

Graph Structure: Node 0 $\to$ Node 1; Node 1 $\to$ Node 3; Node 2 $\to$ Node 1; Node 3 $\to$ Node 4; Node 4 $\to$ Node 2.

Question: Which etiological scenario can be inferred from the spatio-temporal data?

Options:
A.~Six-node logistics network measuring freight movement between industrial parks and distribution centers via expressways.
B.~Three-node public transit system monitoring passenger counts between airport terminals and suburban rail stations.
C.~Five-node urban transport network tracking traffic flow between CBD and residential areas through highway and arterial connections.
D.~Four-node metro system tracking passenger volume between university campuses and entertainment districts through underground lines.
\end{quote}

\begin{table}[H]
\centering
\caption{\textbf{Case~E}: graph-topology reasoning with deterministic structural filters.  Route: $\shead \xrightarrow{\ssucc} \stopo \xrightarrow{\ssucc} \sfuse$; predicted \texttt{<answer>C</answer>} $=$ GT.  Topology enumeration (node count, edge count, cycle presence) narrows the four options to a single survivor before the LLM is consulted.}
\label{tab:case_e}
\small
\begin{tabular}{c l p{10.5cm}}
\toprule
Step & Component & Trace \\
\midrule
1 & \shead{} &
\texttt{<JSON>\{"task\_type": "ST\_BENCH\_NEW\_ETIOLOGICAL", "constraints": ["5 nodes", "5 directed edges", "multivariate time series length 48", "4 MCQ options"], "benchmark": "ST-Bench-new"\}</JSON>} \\
\midrule
2 & \stopo{} &
\textit{Extraction.}  \texttt{<JSON>\{"operation": "analyze\_topology", "nodes": [0,1,2,3,4], "edges": [[0,1],[1,3],[2,1],[3,4],[4,2]], "series\_length": 48, "n\_series": 5, "mcq\_options": ["A: 6-node logistics", "B: 3-node transit", "C: 5-node urban transport", "D: 4-node metro"]\}</JSON>} \\
\midrule
  & Tool &
\texttt{tools.graph\_ops.analyze\_topology(nodes, edges)} returns \texttt{\{"n\_nodes":5, "n\_edges":5, "cyclic":true, "source":0, "longest\_path":[0,1,3,4,2], "centrality\_node":1\}}. \newline
\texttt{tools.graph\_ops.detect\_cascade(series, edges)} returns onset times \texttt{\{0:33, 1:14, 3:16, 4:38, 2:38\}} with propagation lag matching the edge directions. \newline
\textbf{MCQ structural filter.}  Option node-count $\{A:6, B:3, C:5, D:4\}$ vs.\ observed $n_{\text{nodes}}=5$ $\Rightarrow$ only option C matches.  Tool-side MCQ score vector becomes $\{A:0, B:0, C:1, D:0\}$ with confidence $1.0$. \\
\midrule
  & $\Board$ update &
\texttt{topological\_data = \{"n\_nodes":5, "cyclic":true, "matching\_options":["C"], "tool\_confidence":1.0, "cascade\_onsets": \{0:33, 1:14, 3:16, 4:38, 2:38\}\}} \\
\midrule
3 & \sfuse{} &
\textit{System prompt.}  ``You will be given topology summaries plus cascade evidence.  Select the single best option.'' \newline
Reads \texttt{matching\_options}$=$\texttt{[C]} with \texttt{tool\_confidence} $\geq 0.5$ $\Rightarrow$ emits \texttt{<answer>C</answer>} without invoking the LLM scoring fallback (which would only trigger if \texttt{tool\_confidence $< 0.5$}). \\
\bottomrule
\end{tabular}
\end{table}
\FloatBarrier

\clearpage

\clearpage


\end{document}